\documentclass{article} 

\usepackage[sort]{natbib}

\usepackage{iclr2026_conference,times}

\usepackage{hyperref}
\usepackage{url}

\usepackage{bm}
\usepackage{multirow}
\usepackage{booktabs}
\usepackage{algorithm}
\usepackage{algpseudocode}
\usepackage{axlib}

\title{Missingness Bias Calibration in \\Feature Attribution Explanations}

\author{Shailesh Sridhar \\
University of Pennsylvania \\
\texttt{shai2403@seas.upenn.edu}
\AND
Anton Xue \\
University of Texas at Austin \\
\texttt{anton.xue@utexas.edu}
\AND
Eric Wong \\
University of Pennsylvania \\
\texttt{exwong@seas.upenn.edu}
}

\iclrfinalcopy 
\begin{document}

\maketitle

\begin{abstract}

Popular explanation methods often produce unreliable feature importance scores due to missingness bias, a systematic distortion that arises when models are probed with ablated, out-of-distribution inputs.
Existing solutions treat this as a deep representational flaw that requires expensive retraining or architectural modifications.
In this work, we challenge this assumption and show that missingness bias can be effectively treated as a superficial artifact of the model's output space.
We introduce MCal, a lightweight post-hoc method that corrects this bias by fine-tuning a simple linear head on the outputs of a frozen base model.
Surprisingly, we find this simple correction consistently reduces missingness bias and is competitive with, or even outperforms, prior heavyweight approaches across diverse medical benchmarks spanning vision, language, and tabular domains.

\end{abstract}

\section{Introduction}
\label{sec:introduction}

As black-box deep learning systems are increasingly deployed in high-stakes settings such as medicine, finance, and law, there is increasing demand for reliable and trustworthy model explanations.
A common approach for explaining model predictions is to use feature attribution methods, which assign importance scores to input features based on their influence on the output.
Popular methods, such as LIME~\citep{ribeiro2016should} and SHAP~\citep{lundberg2017unified}, estimate these scores by perturbing the input, typically by ablating selected features and measuring the change in prediction.
Because true feature removal is often infeasible (e.g., one cannot physically delete image pixels or omit words from tokenized sequences), attribution methods approximate removal by substituting the selected features with default or placeholder values, such as black pixels or special tokens~\citep{ancona2017towards,sundararajan2017axiomatic}.

These substitutions often result in out-of-distribution inputs that deviate significantly from the model's training data, inducing a systematic distortion in predictions known as \textit{missingness bias}~\citep{hooker2019benchmark,hase2021out,jain2022missingness}.
Such bias can severely undermine the reliability of explanations.
As illustrated in~\cref{fig:brain_mb_example}, a classifier that accurately detects a brain tumor from clean inputs fails to do so when irrelevant features are masked, demonstrating how seemingly innocuous ablations can corrupt model behavior.
Since perturbation-based attributions are derived directly from these corrupted predictions, their reliability is fundamentally compromised, leading to inconsistent feature importance scores~\citep{hooker2019benchmark,duan2024evaluation,goldwasser2024provably}.
This also opens the door to adversarial manipulation: malicious actors can exploit this vulnerability to design deceptive models that obscure their use of sensitive attributes such as race or gender~\citep{slack2020fooling,joe2022exploiting,koyuncu2024exploiting}.

A variety of mitigation strategies have been proposed to address missingness bias.
\textit{Replacement-based} methods aim to reduce distributional shift by imputing masked features with more realistic content~\citep{sturmfels2020visualizing,chang2018explaining,agarwal2020explaining,kim2020interpretation}.
\textit{Training-based} methods fine-tune or retrain the model to better handle ablations~\citep{hooker2019benchmark,hase2021out,rong2022consistent,park2024geometric}, while \textit{architecture-based} approaches embed robustness directly into the model via structural design changes~\citep{jain2022missingness,balasubramanian2023towards}.

However, these strategies are often impractical.
Replacement-based methods are usually specialized to specific domains (e.g., text~\citep{kim2020interpretation}) or might require training model-specific imputations~\citep{chang2018explaining}.
On the other hand, training-based solutions require intensive engineering and computing resources, while architecture-based modifications require a deep understanding of model internals.
Moreover, it is also increasingly common that models are complete black boxes, such as when interacting with API-based LLM providers.

In this work, we question whether such complex interventions are necessary.
We investigate a simple yet surprisingly powerful strategy for mitigating missingness bias:
finetuning a linear head on the outputs of a frozen base model.
This approach, which we call \textbf{MCal}, is \textit{lightweight}, \textit{model-agnostic}, and \textit{post-hoc}.
It is significantly cheaper in implementation effort than training-based methods, does not require model-specific adaptations like architecture-based and replacement-based methods, and only needs access to the model's output logits.
In the following, we summarize the development of MCal and our contributions.

\begin{figure}[t]

\centering









\begin{minipage}{0.25\linewidth}
    \centering
    \includegraphics[width=1.2in]{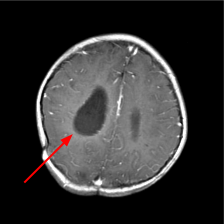}

    Tumor \cmark{}
\end{minipage}%
\qquad
\begin{minipage}{0.25\linewidth}
    \centering
    \includegraphics[width=1.2in]{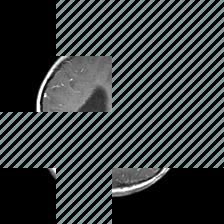}

    Tumor \cmark{}
\end{minipage}%
\qquad
\begin{minipage}{0.25\linewidth}
    \centering
    \includegraphics[width=1.2in]{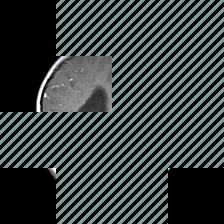}

    Healthy \xmark{}
\end{minipage}%

\caption{\textbf{Removing irrelevant features can cause a misdiagnosis.}
A fine-tuned ViT~\citep{dosovitskiy2020image} correctly predicts ``tumor'' on the clean image (left) and a subset of the relevant features (middle).
However, masking irrelevant features flips the prediction to ``healthy'', despite the tumor remaining visible (right).
For visualization, gray stripes denote zero-valued pixels, and images are contrast-boosted.
}
\label{fig:brain_mb_example}
\end{figure}

\paragraph{A New Perspective on Missingness Bias.}
We find that missingness bias, a problem often treated as a deep representational flaw, can be effectively mitigated with a simple post-hoc correction in the model's output space.
This finding suggests the bias is often a superficial artifact, challenging the prevailing assumption that expensive retraining or architectural modifications are necessary.

\paragraph{A Lightweight Method with Theoretical Guarantees.}
We formalize this approach as MCal, a lightweight calibrator that is highly efficient to optimize (\cref{sec:method}).
Furthermore, our simple formulation provides theoretical guarantees of convergence to a globally optimal solution, ensuring a level of stability and reproducibility rare for deep learning interventions.

\paragraph{A Strong and Practical Baseline.}
We demonstrate MCal's effectiveness across diverse models and data modalities, where it is often competitive with heavyweight approaches (\cref{sec:experiments}).
This establishes a strong and practical baseline that can be immediately adopted by researchers and practitioners to improve the reliability of their explanations.

\section{Understanding Missingness Bias}
\label{sec:background}

Perturbation-based feature attribution methods like LIME~\citep{ribeiro2016should} and SHAP~\citep{lundberg2017unified} evaluate models on inputs with ablated features, typically replaced by fixed baseline values (e.g., zero-vectors or mean-pixel values).
However, because these synthetic inputs often fall outside the model's training distribution, they can induce systematic prediction distortions, a phenomenon known as \textit{missingness bias}.
This section provides a background on this bias and its consequences for explanation reliability.

\begin{figure}[t]
\centering









\includegraphics[width=0.95\linewidth]{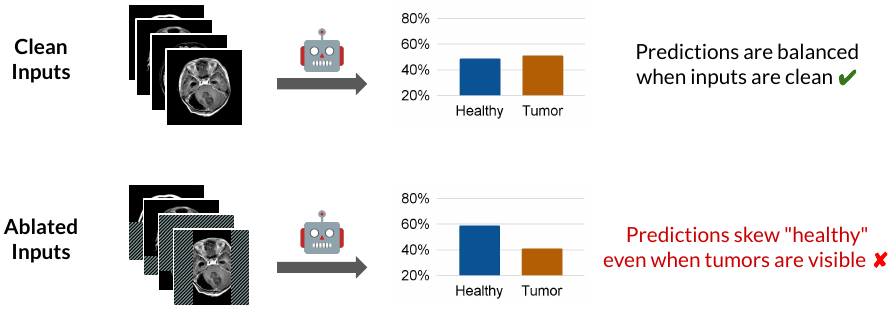}

\caption{\textbf{Feature ablations induce class distribution shifts.}
Masking non-critical regions skews predictions towards the ``healthy'' class, even when tumors remain visible.
This effect, known as \textit{missingness bias}, causes the model to misclassify inputs that retain relevant features, and undermines the reliability of feature attribution explanations.
}
\label{fig:mb_distribution_shift_main_paper}
\end{figure}

\subsection{Pathology: Symptoms and Measurements}

The effects of missingness bias are not merely statistical curiosities; they manifest as tangible failures that undermine the reliability of explanation methods.

\paragraph{Systematic Skew in Predictions.}
The most direct failure mode of missingness bias is a systematic skew in model predictions~\citep{jain2022missingness}.
As illustrated in~\cref{fig:mb_distribution_shift_main_paper}, the model's accuracy degradation is not random but systematic: it develops a consistent bias towards one class (in this case, ``healthy'') even when the core evidence for the correct class remains visible.
This failure mode is particularly pernicious, persisting even when we selectively avoid masking the central image patches most likely to contain the tumor.

\paragraph{Unreliable Feature Attributions.}
Another consequence of this degraded accuracy is that any feature attributions derived from the model are fundamentally unreliable.
If a model's predictions are incorrect on ablated inputs, the importance scores computed from these predictions cannot be trusted to reflect the model's true reasoning.
Empirical findings support this; for instance,~\citet{jain2022missingness} show that feature importance scores from models with high missingness bias fail standard robustness tests such as top-k removal.
Prior work has also shown that minor changes to the ablation process can yield vastly different explanations, suggesting they reflect perturbation artifacts rather than genuine model logic~\citep{hooker2019benchmark}.

\paragraph{Quantifying Missingness Bias.}
Many feature attribution methods operate under the assumption that feature ablation is a neutral act of intervention intended to simulate the removal of information~\citep{sundararajan2017axiomatic,sturmfels2020visualizing}.
When a model's behavior deviates from this expected neutrality, the resulting shift in its aggregate predictive distribution serves as a direct measure of missingness bias.
This shift is typically quantified as the distribution shift between the class frequencies on the clean data distribution \(\mcal{D}\) versus the ablated data distribution \(\mcal{D}'\)~\citep{jain2022missingness,balasubramanian2023towards}:
\begin{equation} \label{eqn:mb_def}
    \msf{MissingnessBias}(f)
    = D_{\mrm{KL}} \parens*{
        \expval{x' \sim \mcal{D}'} \msf{Class}(f(x'))
        \,\big\|\,
        \expval{x \sim \mcal{D}} \msf{Class}(f(x))
    },
\end{equation}
where \(\mcal{D}'\) is the distribution of inputs where each feature is i.i.d. ablated with some given probability, and let \(\msf{Class}(f(x))\) be the one-hot vector representation of the class predicted by \(f\) on \(x\).
The above can then be understood as a measure of information-theoretic ``surprise'' when \(f\) is evaluated on \textit{unbiased} ablations, supposing only knowledge of its behavior on clean inputs.
In particular,~\citet{jain2022missingness} specifically introduces this to measure missingness bias, rather than of adjacent phenomena, such as prediction sensitivity with respect to top-\(k\) feature selections~\citep{hase2021out}.

\subsection{The Challenge of Mitigation}
\label{sec:challenge_mitigation}

A variety of strategies have been proposed to address missingness bias, which can be broadly categorized as follows:

\begin{itemize}
    \item \textit{Replacement-based.} These methods aim to make ablated inputs appear more in-distribution.
    Beyond simple values (e.g., zero and mean-valued~\citep{hase2021out}), more complex variants include marginalization, which averages outputs over plausible replacement values~\citep{kim2020interpretation,frye2020shapley,chirkova2023should,vo2024explainability,haug2021baselines}, random noising~\citep{rong2022consistent}, and generative modeling, which uses a secondary model to in-paint realistic content~\citep{chang2018explaining,agarwal2020explaining}.
    However, these approaches are often complex and can introduce their own artifacts.

    \item \textit{Training-based.} This approach treats feature ablations as a form of data augmentation.
    Methods like ROAR~\citep{hooker2019benchmark} and GOAR~\citep{park2024geometric} retrain or fine-tune the model on masked inputs to align its train and test distributions.
    Although effective at building robust representations, this strategy is computationally expensive and only possible when the model can be modified.
     
    \item \textit{Architecture-based.} These methods embed robustness directly into the model's design.
    For example, modified vision transformers~\citep{jain2022missingness,dosovitskiy2020image} and CNNs~\citep{balasubramanian2023towards} can be altered to use dedicated mask tokens or explicitly suppress the influence of ablated regions.
    However, these changes are often non-trivial, architecture-specific, and not generalizable.
    
\end{itemize}

While often effective, the high cost and complexity of these methods make them impractical for many modern use cases, especially those involving large-scale, pre-trained foundation models.
Furthermore, such approaches are entirely infeasible when working with API-based models that do not permit retraining or architectural changes.
This gap highlights the need for a practical, lightweight, and model-agnostic approach to mitigating missingness bias that we introduce next.

\section{MCal: A Lightweight Calibrator for Missingness Bias}
\label{sec:method}

\begin{figure}[t]
\centering

\includegraphics[height=0.9in]{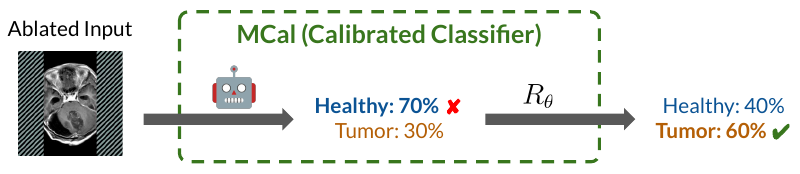}

\caption{\textbf{MCal corrects class distribution shifts induced by input ablations.}
The model initially predicts ``healthy'' from the ablated input.
MCal applies a learned transformation \(R_\theta\) to adjust the output probabilities, thereby restoring alignment with expected class distributions.
This calibration method is model-agnostic, requiring only the classifier's output probabilities of each class.
}
\label{fig:mcal_pipeline}
\end{figure}

Having established the pathology of missingness bias and the practical limitations of existing heavyweight solutions, we now introduce our method.
We propose \textbf{MCal}, a lightweight, post-hoc correction that is surprisingly effective at mitigating missingness bias.

\subsection{Architecture and Optimization}
The calibration process is illustrated in~\cref{fig:mcal_pipeline}.
A base classifier \(f: \mbb{R}^n \to \mbb{R}^m\) first processes an input \(x\) to output the \textit{raw logits} \(z = f(x)\).
A calibrator \(R_\theta : \mbb{R}^m \to \mbb{R}^m\) then transforms the raw logits into the \textit{calibrated logits} \(R_\theta (z)\).
Specifically, we implement this as an affine transform:
\begin{equation}
    R_\theta (z) = Wz + b,
\end{equation}
where the calibrator is parametrized by \(\theta = (W, b)\), with \(W \in \mbb{R}^{m \times m}\) and \(b \in \mbb{R}^m\).
To fit the calibrator, we use a standard cross-entropy objective that aligns the calibrated prediction on an ablated input with the base model's prediction on the clean input:
\begin{equation} \label{eqn:finetune_loss}
    \mathcal{L}(\theta) = \expval{(x, x') \sim \mcal{D}} \mrm{CrossEntropy}\bracks*{R_\theta (f(x')), \msf{Class}(f(x)) },
\end{equation}
where \((x, x') \sim \mcal{D}\) are samples of a clean input \(x\) and its ablated version \(x'\), and \(\msf{Class}(f(x))\) denotes the one-hot prediction on the clean input.

Our approach is deliberately minimalist, prioritizing efficiency without compromising performance.
We apply a standard cross-entropy objective, identical to that used in heavyweight retraining methods~\citep{hooker2019benchmark}, but only to a lightweight matrix-scaling calibrator~\citep{guo2017calibration}.
This design is highly efficient, with orders of magnitude fewer parameters (\(m^2 + m\)) than fine-tuning or even parameter-efficient methods like LoRA~\citep{hu2022lora}.
Our experiments in~\cref{sec:experiments} confirm that this minimalist approach is, in fact, sufficient to yield competitive performance with more engineering-intensive approaches like retraining the model or architecture modifications.
Furthermore, this simple design also comes with strong theoretical guarantees on its optimization process, which we detail next.

\subsection{Theoretical Guarantees and Geometric Interpretation}

Our affine parametrization of \(R_\theta\) means that standard gradient-based optimization will provably converge to an optimal solution, which we formalize as follows.

\begin{theorem}[Guaranteed Optimal Convergence]
\label{thm:mcal_optimal}
The MCal objective \(\mcal{L}(\theta)\) is convex in \(\theta\).
\end{theorem}
\begin{proof}
The function \(\mcal{L}(\theta)\) is convex in \(\theta\), as it is a composition of the convex cross-entropy loss and an affine transformation.
Because local minimums are also global minimums for convex functions, standard gradient-based optimization (e.g., SGD, Adam) will converge to an optimal solution.
\end{proof}

The importance of this guarantee is twofold.
First, it ensures reproducibility and stability: the optimization process is guaranteed to converge to the same optimal solution, reducing the need for extensive hyperparameter sweeps or random seed searches.
Second, it provides a strong assurance of quality, guaranteeing that the resulting calibrator is a globally optimal affine correction for the given data.

\paragraph{Geometric Interpretation.}
MCal also has a clear geometric interpretation, visualized in~\cref{fig:simplex_calibration}.
The uncalibrated outputs form biased point clouds on the probability simplex, with the Class A cluster pulled towards the Class B vertex, leading to systematic misclassification.
MCal learns an optimal affine transformation in the logit space that rotates, scales, and shifts these distributions.
This untangles the clouds and pushes them towards their correct vertices.
\cref{thm:mcal_optimal} guarantees that this correction is globally optimal for our parametrization.

\begin{figure}[t]
    \centering
    \begin{minipage}{0.95\textwidth}
    \centering
    \begin{tikzpicture}[remember picture]
        
        \node (leftfig) {
            \subcaptionbox*{Uncalibrated Accuracy: 59.33\%\label{fig:simplex_a}}{%
                \includegraphics[height=1.2in]{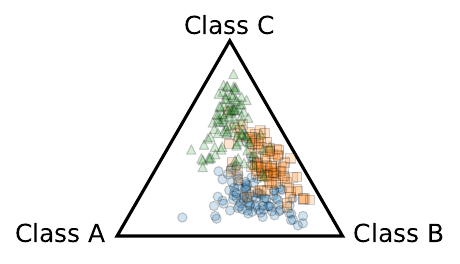}%
            }
        };

        \node at ([xshift=1.6in]leftfig.east) (rightfig) {
            \subcaptionbox*{Calibrated Accuracy: 93.00\% \label{fig:simplex_b}}{%
                \includegraphics[height=1.2in]{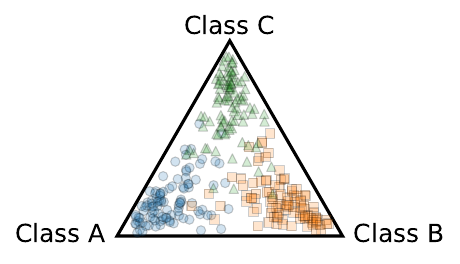}%
            }
        };

    \end{tikzpicture}
    
    \begin{tikzpicture}[overlay, remember picture]
        \draw[->, thick, shorten >=4pt, shorten <=4pt, color=gray, bend left=10]
            (leftfig.east) to node[above, font=\small, black] {MCal} (rightfig.west);
    \end{tikzpicture}
    \end{minipage}


    \caption{\textbf{Geometric intuition of MCal on a synthetic dataset.}
    Missingness bias causes the uncalibrated outputs to shift.
    For instance, the Class A cluster (blue circles) is pulled towards the Class B vertex, leading to systematic misclassification and low accuracy.
    MCal applies an optimal affine transformation to the uncalibrated outputs, correcting the shift and improving accuracy.    
    }


    \label{fig:simplex_calibration}
\end{figure}
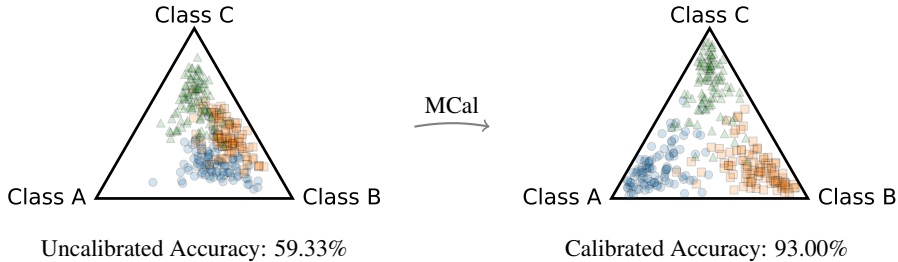

\subsection{Implementation Considerations}
\label{sec:implementation_considerations}

\paragraph{Conditioning on Ablation Rates.}
Our experience shows that the severity of missingness bias is strongly correlated with the fraction of features that are ablated.
To account for this, we recommend using an ``ensemble'' of specialized calibrators, each one fit for a specific ablation rate (e.g., 10\%, 20\%, etc.).
At inference time, we apply the calibrator that was trained for the ablation rate closest to that of the input.
We study the advantage of this ensemble in~\cref{sec:experiments}.

\paragraph{Integration with Explainers.}
As a post-hoc wrapper, MCal is compatible with any perturbation-based explanation method.
The calibrated model, \(\tilde{f}\), can be used as a drop-in replacement for the original model, \(f\), in any existing explainability pipeline.
The resulting feature attributions are then generated from a model that has been explicitly corrected for the missingness bias induced by the explanation method's own perturbation strategy.

{\color{rev1}
\paragraph{Training Set Size and Overfitting.}
Dense parametrizations of \(W\) risk overfitting when the number of parameters exceeds the number of training samples~\citep{guo2017calibration}, which can occur when there are many classes.
In such cases, the training loss may go to zero while test performance does not improve.
We recommend two strategies to mitigate overfitting.
First, one may consider adding a regularization term to the objective.
Second, one may also consider sparse parametrizations, such as taking \(W\) to be a diagonal matrix (also known as ``vector-scaling''), which would reduce the total parameter count to \(O(m)\).
}


\section{Experiments}
\label{sec:experiments}

We now present experiments to validate the impact of missingness bias in explainability, as well as the ability of MCal to mitigate it.
Moreover, we demonstrate that MCal repeatedly outperforms more expensive baselines, such as full retraining and architecture modifications.
Additional details are given in~\cref{sec:additional_experiments}.


\paragraph{Models, Datasets, and Compute.}
We evaluate on a diverse set of medical benchmarks that span vision (Brain MRI~\citep{nickparvar2021mri}, Chest X-ray (CheXpert)~\citep{irvin2019chexpert}, and Breast Cancer Histopathology (BreakHis)~\citep{spanhol2015dataset}),
language (MedQA~\citep{jin2021disease}, MedMCQA~\citep{pal2022medmcqa}), 
and tabular domains (PhysioNet~\citep{haug2021baselines}, Breast Cancer~\citep{wolberg1993breast}, Cardiotocography (CTG)~\citep{campos2000cardiotocography}).
We respectively evaluate on these domains with ViT-B16~\citep{dosovitskiy2020image}, Llama-3.1-8B-Instruct~\citep{meta2024llama3}, and XGBoost~\citep{chen2016xgboost}, which are trained using standard methods.
For compute, we had access to a machine with four NVIDIA H100 NVL GPUs.

\paragraph{Input Ablations and Calibration.}
We say that an input \(x \in \mbb{R}^n\) has ablation rate \(p = k/n\) if \(k\) of its features are ablated.
To evaluate on a tractable range of \(p\), we use \(p \in \{0/16, 1/16, \ldots, 15/16\}\) for vision, \(p \in \{0/10, 1/10, \ldots, 9/10\}\) for language, and \(p \in \{0/10, 1/10, \ldots, 9/10\}\) for tabular, where recall that we recover the clean input at \(p = 0\).
For imputations, we use zero-valued (black) patches for vision, we replace whitespace-separated words with the special string \texttt{UNKWORDS} for language, and we perform mean imputation for tabular data.
For vision specifically, we select \(k\) patches to ablate, regardless of their original values (e.g., some MRI images already have black patches).
Following discussion from~\cref{sec:implementation_considerations}, the \textit{unconditioned} calibrator was fit on inputs where each feature was uniformly ablated with probability \(1/2\), whereas the \textit{conditioned} ensemble has a calibrator fit at each value of \(p\).
All calibrators were optimized using Adam~\citep{kingma2014adam} with a learning rate of \(10^{-3}\) for \(5000\) steps.

\begin{figure*}[t]

\centering

\begin{minipage}{0.24\linewidth}
    \includegraphics[width=1.0\linewidth]{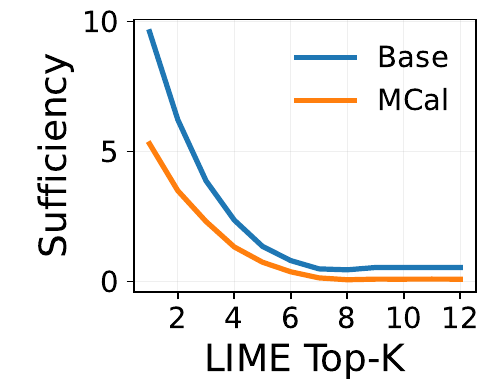}
\end{minipage}%
\begin{minipage}{0.24\linewidth}
    \includegraphics[width=1.0\linewidth]{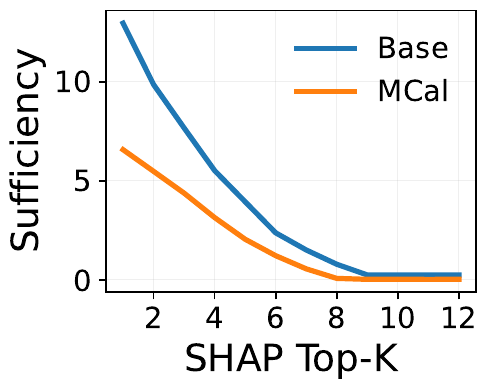}
\end{minipage}%
\begin{minipage}{0.24\linewidth}
    \includegraphics[width=1.0\linewidth]{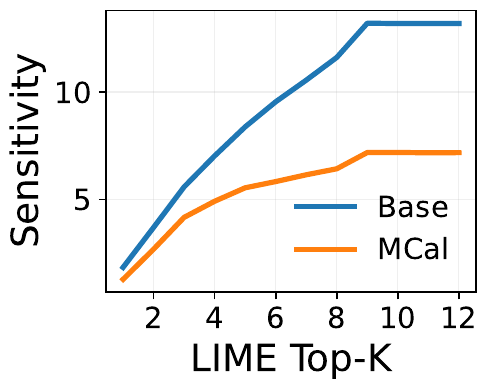}
\end{minipage}%
\begin{minipage}{0.24\linewidth}
    \includegraphics[width=1.0\linewidth]{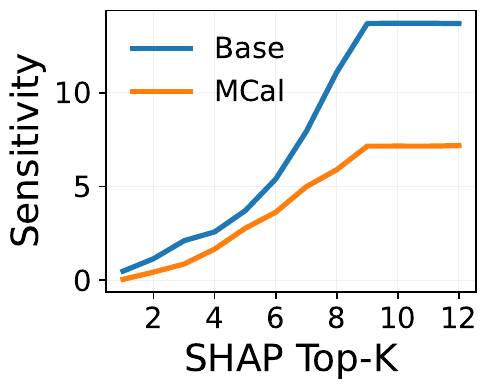}
\end{minipage}%


\caption{
\textbf{Calibrated models have better explanations.}
Compared to an uncalibrated baseline model (Base), LIME and SHAP explanations on MCal-calibrated models have more accurate feature importance scores (sufficiency $\downarrow$).
{\color{rev1}
In addition, calibrated models are also more robust to feature ablations (sensitivity $\downarrow$).
Results are shown for the MRI dataset using an unconditioned calibrator.
}
%
}
\label{fig:main_paper_mcal_other_metrics}
\end{figure*}

\paragraph{Question 1: Do calibrated models lead to better explanations?}
Missingness bias is known to skew the explanation quality of feature attribution methods~\citep{jain2022missingness}.
To that end, we consider how two representative methods, LIME~\citep{ribeiro2016should} and SHAP~\citep{lundberg2017unified}, perform on calibrated vs. uncalibrated models.
These methods output a ranking of each input feature's importance to the model, which we evaluate using the standard \textit{sufficiency} metric~\citep{hase2021out}, detailed in~\cref{sec:other_missingness_metrics}.
Informally, sufficiency measures whether the features identified as important are enough on their own to maintain the model's original prediction confidence (lower values indicate a higher quality ranking).
We report results in~\cref{fig:main_paper_mcal_other_metrics}.

{\color{rev1}\paragraph{Question 2: How does calibration affect model robustness?}
It is generally desirable for models to be robust to feature perturbations, as this can improve generalization and reduce the risk of adversarial behaviors.
To that end, we measure the robustness of the underlying model to the removal (ablation) of features via the \textit{sensitivity} metric, detailed in~\cref{fig:main_paper_mcal_other_metrics}.
We show our results in~\cref{sec:other_missingness_metrics},  which shows that the model is not overly dependent on its top-k features for prediction.
}

\begin{figure*}[t]

\centering

\begin{minipage}{0.32\linewidth}
    \includegraphics[width=1.0\linewidth]{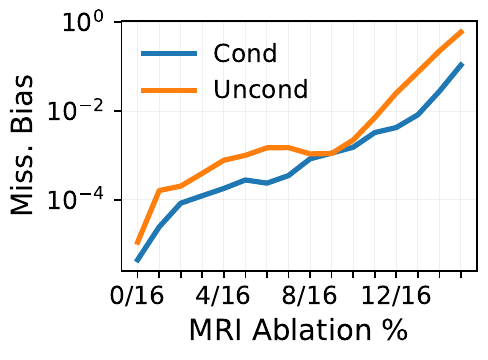}
\end{minipage}%
\begin{minipage}{0.32\linewidth}
    \includegraphics[width=1.0\linewidth]{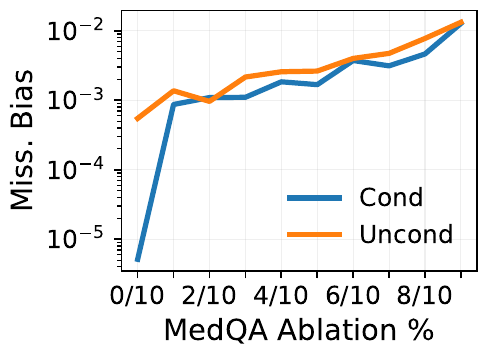}
\end{minipage}%
\begin{minipage}{0.32\linewidth}
    \includegraphics[width=1.0\linewidth]{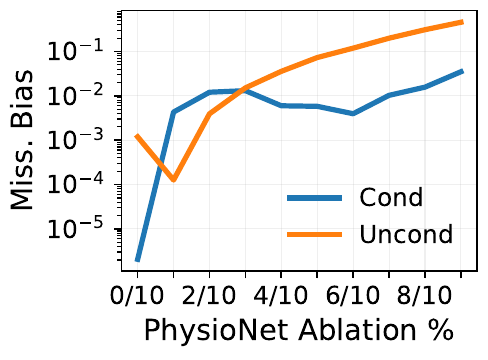}
\end{minipage}%

\caption{\textbf{Conditioning on ablation rate improves MCal.}
Fitting an ensemble of calibrators at a discretized set of ablation rates can help reduce the overall missingness rate, compared to using a single unconditioned calibrator.
(Left) MRI, (Middle) MedQA, (Right) PhysioNet.
}
\label{fig:main_paper_fraction_condition}
\end{figure*}

\paragraph{{\color{rev1}Question 3:} What is the impact of conditioning on feature ablation fractions?}
Rather than fitting a single calibrator, we observe that using an ensemble of calibrators, each conditioned upon a single fraction (ablation rate), can improve performance.
We compare the performance of this conditioning in~\cref{fig:main_paper_fraction_condition}, where we observe an improvement in performance over an unconditioned calibrator.
This is expected, as a model's missingness bias is known to vary with the ablation rate~\citep{jain2022missingness,hooker2019benchmark}, and an ensemble thereby allows each calibrator to specialize to their respective rates.




\begin{table}[t]
\centering
\footnotesize
{\setlength{\tabcolsep}{3pt}
\begin{tabular}[t]{llccccccc}
\toprule
\textbf{} & \textbf{Dataset} & \textbf{Base} & \textbf{Replace} & \textbf{Retrain} & \textbf{Arch} & \textbf{TempCal} & \textbf{PlattCal} & \textbf{MCal} (\cmark) \\
\midrule
\multirow{3}{*}{\textbf{Vision}} 
  & Brain MRI &
    \(1.18\,\mrm{e}{-1}\) & 
    \(1.51\,\mrm{e}{-1}\) & 
    \(\mbf{6.70\,e{-4}}\) & 
    \(1.40\,\mrm{e}{-1}\) & 
    \(1.16\,\mrm{e}{-1}\) & 
    \(1.27\,\mrm{e}{-1}\) & 
    \(7.43\,\mrm{e}{-3}\) \\ 
  & CheXpert &
    \(1.70\,\mrm{e}{-1}\) & 
    \(9.70\,\mrm{e}{-2}\) & 
    \(2.67\,\mrm{e}{-2}\) & 
    \(1.50\,\mrm{e}{-1}\) & 
    \(1.65\,\mrm{e}{-1}\) & 
    \(2.02\,\mrm{e}{-1}\) & 
    \(\mbf{8.82\,e{-3}}\) \\ 
  & BreakHis &
    \(1.87\,\mrm{e}{-1}\) & 
    \(4.20\,\mrm{e}{-1}\) & 
    \(2.19\,\mrm{e}{-2}\) & 
    \(1.54\,\mrm{e}{-1}\) & 
    \(1.86\,\mrm{e}{-1}\) & 
    \(1.66\,\mrm{e}{-1}\) & 
    \(\mbf{4.29\,e{-3}}\) \\ 
\midrule
\multirow{2}{*}{\textbf{Language}} 
  & MedQA &
    \(1.61\,\mrm{e}{-1}\) & 
    \(1.50\,\mrm{e}{-1}\) & 
    \(1.70\,\mrm{e}{-1}\) & 
    \(2.68\,\mrm{e}{-2}\) & 
    \(1.57\,\mrm{e}{-1}\) & 
    \(9.48\,\mrm{e}{-2}\) & 
    \(\mbf{9.44\,e{-4}}\) \\ 
  & MedMCQA &
    \(1.89\,\mrm{e}{-1}\) & 
    \(2.59\,\mrm{e}{-1}\) & 
    \(1.52\,\mrm{e}{-1}\) & 
    \(1.40\,\mrm{e}{-1}\) & 
    \(7.81\,\mrm{e}{-1}\) & 
    \(1.13\,\mrm{e}{-1}\) & 
    \(\mbf{9.01\,e{-3}}\) \\ 
\midrule
\multirow{3}{*}{\textbf{Tabular}} 
  & PhysioNet &
    \(1.17\,\mrm{e}{-1}\) & 
    \(1.20\,\mrm{e}{-1}\) & 
    \(5.59\,\mrm{e}{-3}\) & 
    \(8.14\,\mrm{e}{-2}\) & 
    \(1.17\,\mrm{e}{-1}\) & 
    \(1.19\,\mrm{e}{-1}\) & 
    \(\mbf{5.01\,e{-3}}\) \\ 
  & Breast Cancer &
    \(1.02\,\mrm{e}{-1}\) & 
    \(1.44\,\mrm{e}{-1}\) & 
    \(5.68\,\mrm{e}{-3}\) & 
    \(2.13\,\mrm{e}{-1}\) & 
    \(1.02\,\mrm{e}{-1}\) & 
    \(1.08\,\mrm{e}{-1}\) & 
    \(\mbf{1.92\,e{-5}}\) \\ 
  & CTG &
    \(1.06\,\mrm{e}{-1}\) & 
    \(7.02\,\mrm{e}{-2}\) & 
    \(6.61\,\mrm{e}{-3}\) & 
    \(2.85\,\mrm{e}{-1}\) & 
    \(1.06\,\mrm{e}{-1}\) & 
    \(9.20\,\mrm{e}{-2}\) & 
    \(\mbf{3.35\,e{-3}}\) \\ 
\bottomrule
\end{tabular}
}
\caption{
\textbf{MCal is an effective and lightweight way to reduce missingness bias.}
It repeatedly outperforms more computationally expensive baselines, such as retraining and architecture modification.
We report the KL divergence-based metric in~\cref{eqn:mb_def}.
}
\label{tab:main_paper_baselines}
\end{table}

\paragraph{{\color{rev1}Question 4:} How does MCal compare to the baselines?}

We compare MCal to each of the following prior approaches, which have all been employed in previous work to combat the problem of out-of-distribution inputs.

\begin{itemize}
  \item \textbf{Base}: This is the unmodified, uncalibrated classifier that acts as a reference baseline. 

  
  \item \textbf{Replacement-based (Replace)}: 
  Our implementation of replacement-based mitigation is inspired from~\citet{hase2021out}.
  In particular, for vision, we use the channel-wise mean pixel value of the clean dataset \citep{carter2021overinterpretation}.  
  For language, we drop tokens from the sequence so that the ablated token sequence is shorter in length than the clean one \citep{hase2021out}.
  For tabular, we perform mean imputation.

  \item \textbf{Training-based approaches (Retrain)}: Models are fine-tuned on ablated inputs, where each feature (patch, token) is uniformly ablated with probability \(1/2\).
  
  \item \textbf{Architectural-based (Arch)}: We perform a non-trivial modification of ViT to accept attention masks as in~\citet{jain2022missingness}. 
  For models with architectural support for missing features, we use those: e.g., attention masking in Llama-3 and native support for \texttt{NaN} in XGBoost.

  %

  \item \textbf{Standard calibration (TempCal, PlattCal)}:
  We additionally consider existing calibration-based methods from literature, particularly temperature (TempCal) and vector-scaling Platt calibration (PlattCal), as described in~\citet{guo2017calibration}.



\end{itemize}
We report in~\cref{tab:main_paper_baselines} the average of values from the ensemble of conditioned calibrators.
We found that MCal is often superior even to more computationally and engineering-intensive baselines, such as model retraining and ViT architecture modifications.
In support of our earlier claims, we also observe that MCal outperforms both temperature and Platt calibration.
Replacement-based methods have inconsistent performance, which aligns with known observations on their sensitivity to imputation values.
Finally, we note that architecture-native support for missing features may, in fact, exacerbate missingness bias, as seen in XGBoost on the Breast Cancer and CTG datasets.



\begin{figure*}

\centering

\begin{minipage}{0.32\linewidth}
    \includegraphics[width=1.0\linewidth]{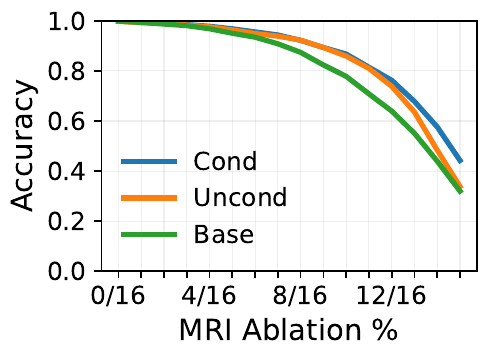}
\end{minipage}%
\begin{minipage}{0.32\linewidth}
    \includegraphics[width=1.0\linewidth]{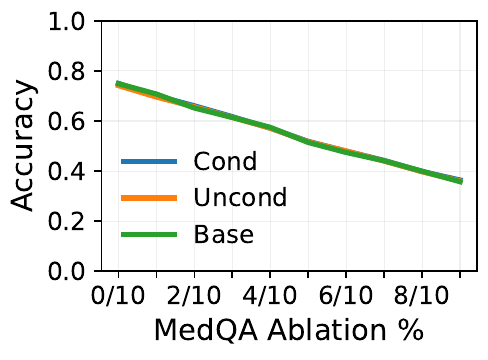}
\end{minipage}%
\begin{minipage}{0.32\linewidth}
    \includegraphics[width=1.0\linewidth]{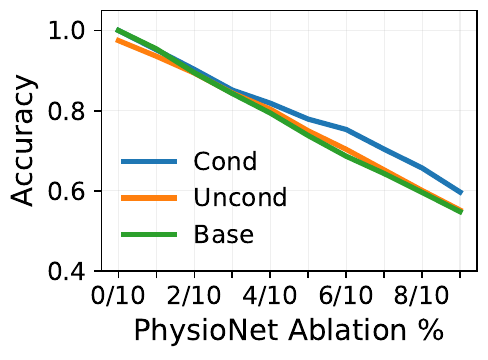}
\end{minipage}%

\caption{\textbf{MCal does not harm classifier accuracy.}
{\color{rev1}
Across different ablation levels, the accuracy of uncalibrated vs. calibrated classifiers is comparable.
This also holds when the input is clean (ablation fraction zero).
(Left) MRI, (Middle) MedQA, (Right) PhysioNet.
}
}
\label{fig:main_paper_mcal_accuracy}
\end{figure*}

\paragraph{{\color{rev1}Question 5:} How does MCal affect classifier accuracy?}
MCal fundamentally alters a pretrained base classifier \(f\) into \(\tilde{f}\), which is then deployed to downstream applications.
Importantly, the accuracy of \(\tilde{f}\) must remain high, even when it is optimized on ablated images~\cref{eqn:finetune_loss}.
We show in~\cref{fig:main_paper_mcal_accuracy} that this is indeed the case: we compare the uncalibrated base model against both ablation rate-conditioned and unconditioned calibrators.
We observe that both forms of calibration improve classifier accuracy at all ablation rates, where we recall that the clean image is obtained at an ablation rate of zero.
Aligning with earlier findings, we see that the conditioned calibrator outperforms the unconditioned calibrator.




\section{Related Work}
\label{sec:related}

\paragraph{Missingness Bias in Explainability.}
Missingness bias~\citep{jain2022missingness} denotes the systematic distortions that arise when attribution methods ``remove'' features via ablations, e.g., with black pixels, zero-valued embeddings, or special \texttt{[MASK]} tokens.
Such ablated inputs are often out-of-distribution with respect to the model's training distribution, which can result in erratic predictions, inflated confidences, and unstable feature importance scores~\citep{hooker2019benchmark,vo2024explainability}.
In particular, importance scores can vary drastically with the chosen replacement technique~\citep{sturmfels2020visualizing,haug2021baselines} and can even be exploited adversarially~\citep{slack2020fooling}.
Consequently, feature-based explanations commonly reflect ablation artifacts rather than genuine model reasoning~\citep{hase2021out}, which risks eroding trust in high-stakes settings.
In addition to the methods described earlier in~\cref{sec:challenge_mitigation}, there are several benchmarks related to missingness bias~\citep{liu2021synthetic,hesse2023funnybirds,duan2024evaluation}.

\paragraph{Calibration Methods.}
A calibration method post-hoc rescales the logits or probabilities of a model prediction without modifying the underlying model weights.
Classic techniques include binning~\citep{zadrozny2001obtaining}, Platt scaling~\citep{platt1999probabilistic}, and temperature scaling~\citep{guo2017calibration}.
This is often used to improve and calibrate model predictions under input distribution shift, such as in autonomous driving~\citep{tomani2021post}, healthcare~\citep{shashikumar2023unsupervised}, and LLMs~\citep{kumar2022calibrated}.
A closely related work is~\citet{decker2025improving}, which addresses perturbation-induced miscalibration but focuses on confidence rather than the systematic class-skew.


\paragraph{Robust and Reliable Explanations.}
There is much interest in the development of robust explanations for machine learning models.
Notable efforts include the development of benchmarks for explanations, particularly feature attribution methods~\citep{adebayo2018sanity,adebayo2022post,agarwal2022openxai,dinu2020challenging,jin2024fix,havaldar2025t,kindermans2019reliability,nauta2023anecdotal,rong2022consistent,duan2024evaluation,zhou2022feature}.
There is also interest in formally certifying explanations~\citep{bassan2023towards,bassan2025explaining,xue2023stability,jin2025probabilistic,lin2023robustness,you2025probabilistic,bassan2025makes}.
Other efforts, such as this work, involve adapting classifiers to be more robust to input ablations in feature attributions.

\section{Discussion, Future Directions, and Conclusion}
\label{sec:discussion}



\paragraph{Calibration Design.}
While other calibrator parametrizations are viable, any non-convex parametrization of the objective risks losing guarantees of optimality convergence.
In turn, this risks introducing undesirable behavior, such as sensitivity to the initialization of calibrator parameters.
Additionally, observe that the measure of missingness bias (\cref{eqn:mb_def}) is different than the calibrator optimization objective.
This is because the missingness bias measure is not differentiable due to the one-hot \(\msf{Class}\) function, which motivated us to search for reasonable alternatives, e.g., the standard cross-entropy objective in classification.
While it would be interesting to explore, for instance, differentiable relaxations of~\cref{eqn:mb_def}, we leave this to future work.

{\color{rev1}
\paragraph{Missingness Bias and Data Variability.}
Visual medical datasets, e.g., Chest X-rays, often exhibit lower variability than image datasets like ImageNet~\citep{deng2009imagenet}.
Despite this, missingness bias can persist in models trained on datasets such as ImageNet~\citep{jain2022missingness}.
However, it is not known how missingness bias varies as a function of both dataset variability and model architecture.
}

\paragraph{Beyond Explainability.}
Missingness bias is a fundamental risk when evaluating feature subsets on a model that is not explicitly designed to handle missing data.
While we are primarily motivated by challenges in explainability, this work has broader applications.
In vision, model evaluation with masked images is a standard practice.
In language modeling, a token's embedding is often dependent on its position, meaning that ablations are position-sensitive, whether via the attention mask, subsetting the input sequence, or replacement with special \texttt{[MASK]} tokens.



\paragraph{Limitations.}
MCal requires access to a collection of clean and ablated prediction logits, which may not always be available, such as for some API-based LLMs.
Even then, gradient-based optimization is only guaranteed to converge to global optimality under certain parameterizations of the calibrator.
Overfitting is also a potential risk, particularly in settings with a large number of possible classes (e.g., a language model's vocabulary size), in which case regularization is warranted.
Furthermore, MCal is only intended to mitigate missingness bias, and other forms of bias in the model and data may still be propagated.
{\color{rev1}
While our experiments show that linear corrections on the logits suffice to mitigate missingness bias, and hence our use of the description ``superficial'', it may be the case that for missingness bias in certain model classes, it is harder to mitigate in this manner.
}

\paragraph{Future Work.}
One direction is to investigate the theoretical guarantees and empirical performance of different calibrator parametrizations, such as a one-layer feedforward network instead of an affine transform.
Another extension is to broaden our study on the performance of calibrated classifiers in explainability, such as with respect to the explanation methods and metrics surveyed in~\cref{sec:related}.
It would be interesting to explore methods for mitigating missingness bias when prediction logits are not available, a common restriction for API-based large language models.
Additionally, the idea of calibration may also be extended to other instances of domain shift and out-of-distribution inputs, which are prevalent throughout machine learning literature.

\paragraph{Conclusion.}
Missingness bias threatens the reliability of popular explanation methods and techniques, a problem magnified by the increasing impracticality of existing engineering-intensive solutions.
To overcome this, we introduce MCal, a lightweight calibration method that requires only a collection of clean and ablated prediction pairs.
We demonstrate that a simple, affine parametrization of the calibrator offers strong theoretical guarantees while achieving empirical performance that often outperforms more expensive baselines.
In summary, MCal is an efficient, model-agnostic calibration scheme that improves the reliability of popular feature-based explanation methods.


\paragraph{Ethics Statement.}
This work presents a method for improving the reliability of feature-based explanation methods.
Our intended audience includes researchers and practitioners interested in explainability.
While there may be potential for misuse, we do not believe that the contents of this paper warrant concern.

\paragraph{Reproducibility Statement.}
All code and experiments for this paper are available at:
\begin{center}
\url{https://github.com/ShaileshSridhar2403/MCal}
\end{center}

\bibliography{references}

@article{jain2022missingness,
  title={Missingness bias in model debugging},
  author={Jain, Saachi and Salman, Hadi and Wong, Eric and Zhang, Pengchuan and Vineet, Vibhav and Vemprala, Sai and Madry, Aleksander},
  journal={arXiv preprint arXiv:2204.08945},
  year={2022}
}

@inproceedings{ribeiro2016should,
  title={" Why should i trust you?" Explaining the predictions of any classifier},
  author={Ribeiro, Marco Tulio and Singh, Sameer and Guestrin, Carlos},
  booktitle={Proceedings of the 22nd ACM SIGKDD international conference on knowledge discovery and data mining},
  pages={1135--1144},
  year={2016}
}

@article{lundberg2017unified,
  title={A unified approach to interpreting model predictions},
  author={Lundberg, Scott M and Lee, Su-In},
  journal={Advances in neural information processing systems},
  volume={30},
  year={2017}
}

@inproceedings{guo2017calibration,
  title={On calibration of modern neural networks},
  author={Guo, Chuan and Pleiss, Geoff and Sun, Yu and Weinberger, Kilian Q},
  booktitle={International conference on machine learning},
  pages={1321--1330},
  year={2017},
  organization={PMLR}
}

@misc{meta2024llama3,
  title={Llama 3 Model Card},
  author={AI@Meta},
  year={2024},
  url={https://github.com/meta-llama/llama3/blob/main/MODEL_CARD.md}
}

@article{hooker2019benchmark,
  title={A benchmark for interpretability methods in deep neural networks},
  author={Hooker, Sara and Erhan, Dumitru and Kindermans, Pieter-Jan and Kim, Been},
  journal={Advances in neural information processing systems},
  volume={32},
  year={2019}
}

@article{hase2021out,
  title={The out-of-distribution problem in explainability and search methods for feature importance explanations},
  author={Hase, Peter and Xie, Harry and Bansal, Mohit},
  journal={Advances in neural information processing systems},
  volume={34},
  pages={3650--3666},
  year={2021}
}

@article{agarwal2020explaining,
  title={Explaining image classifiers by removing input features using generative models}, 
  author={Chirag Agarwal and Anh Nguyen},
  journal={arXiv preprint arXiv:1910.04256},
  year={2020} 
}

@inproceedings{kim2020interpretation,
  title={Interpretation of NLP models through input marginalization},
  author={Kim, Siwon and Yi, Jihun and Kim, Eunji and Yoon, Sungroh},
  booktitle={Proceedings of the 2020 Conference on Empirical Methods in Natural Language Processing (EMNLP)},
  pages={3154--3167},
  year={2020}
}

@inproceedings{irvin2019chexpert,
  title={Chexpert: A large chest radiograph dataset with uncertainty labels and expert comparison},
  author={Irvin, Jeremy and Rajpurkar, Pranav and Ko, Michael and Yu, Yifan and Ciurea-Ilcus, Silviana and Chute, Chris and Marklund, Henrik and Haghgoo, Behzad and Ball, Robyn and Shpanskaya, Katie and others},
  booktitle={Proceedings of the AAAI conference on artificial intelligence},
  volume={33},
  pages={590--597},
  year={2019}
}

@article{dosovitskiy2020image,
  title={An image is worth 16x16 words: Transformers for image recognition at scale},
  author={Dosovitskiy, Alexey and Beyer, Lucas and Kolesnikov, Alexander and Weissenborn, Dirk and Zhai, Xiaohua and Unterthiner, Thomas and Dehghani, Mostafa and Minderer, Matthias and Heigold, Georg and Gelly, Sylvain and others},
  journal={arXiv preprint arXiv:2010.11929},
  year={2020}
}

@inproceedings{sundararajan2017axiomatic,
  title={Axiomatic attribution for deep networks},
  author={Sundararajan, Mukund and Taly, Ankur and Yan, Qiqi},
  booktitle={International conference on machine learning},
  pages={3319--3328},
  year={2017},
  organization={PMLR}
}

@article{sturmfels2020visualizing,
  title={Visualizing the impact of feature attribution baselines},
  author={Sturmfels, Pascal and Lundberg, Scott and Lee, Su-In},
  journal={Distill},
  volume={5},
  number={1},
  pages={e22},
  year={2020}
}

@inproceedings{deng2009imagenet,
  title={Imagenet: A large-scale hierarchical image database},
  author={Deng, Jia and Dong, Wei and Socher, Richard and Li, Li-Jia and Li, Kai and Fei-Fei, Li},
  booktitle={2009 IEEE conference on computer vision and pattern recognition},
  pages={248--255},
  year={2009},
  organization={Ieee}
}

@article{spanhol2015dataset,
  title={A dataset for breast cancer histopathological image classification},
  author={Spanhol, Fabio A and Oliveira, Luiz S and Petitjean, Caroline and Heutte, Laurent},
  journal={Ieee transactions on biomedical engineering},
  volume={63},
  number={7},
  pages={1455--1462},
  year={2015},
  publisher={IEEE}
}

@inproceedings{hesse2023funnybirds,
  title={Funnybirds: A synthetic vision dataset for a part-based analysis of explainable ai methods},
  author={Hesse, Robin and Schaub-Meyer, Simone and Roth, Stefan},
  booktitle={Proceedings of the IEEE/CVF International Conference on Computer Vision},
  pages={3981--3991},
  year={2023}
}

@article{liu2021synthetic,
  title={Synthetic benchmarks for scientific research in explainable machine learning},
  author={Liu, Yang and Khandagale, Sujay and White, Colin and Neiswanger, Willie},
  journal={arXiv preprint arXiv:2106.12543},
  year={2021}
}

@article{ancona2017towards,
  title={Towards better understanding of gradient-based attribution methods for deep neural networks},
  author={Ancona, Marco and Ceolini, Enea and {\"O}ztireli, Cengiz and Gross, Markus},
  journal={arXiv preprint arXiv:1711.06104},
  year={2017}
}

@article{chang2018explaining,
  title={Explaining image classifiers by counterfactual generation},
  author={Chang, Chun-Hao and Creager, Elliot and Goldenberg, Anna and Duvenaud, David},
  journal={arXiv preprint arXiv:1807.08024},
  year={2018}
}

@inproceedings{zadrozny2001obtaining,
  title={Obtaining calibrated probability estimates from decision trees and naive Bayesian classifiers},
  author={Zadrozny, Bianca and Elkan, Charles},
  booktitle={Proceedings of the Eighteenth International Conference on Machine Learning},
  pages={609--616},
  year={2001}
}

@article{platt1999probabilistic,
  title={Probabilistic outputs for support vector machines and comparisons to regularized likelihood methods},
  author={Platt, John and others},
  journal={Advances in large margin classifiers},
  volume={10},
  number={3},
  pages={61--74},
  year={1999},
  publisher={Cambridge, MA}
}

@inproceedings{duan2024evaluation,
  title={On the Evaluation Consistency of Attribution-Based Explanations},
  author={Duan, Jiarui and Li, Haoling and Zhang, Haofei and Jiang, Hao and Xue, Mengqi and Sun, Li and Song, Mingli and Song, Jie},
  booktitle={European Conference on Computer Vision},
  pages={206--224},
  year={2024},
  organization={Springer}
}

@article{jin2025probabilistic,
  title={Probabilistic Stability Guarantees for Feature Attributions},
  author={Jin, Helen and Xue, Anton and You, Weiqiu and Goel, Surbhi and Wong, Eric},
  journal={arXiv preprint arXiv:2504.13787},
  year={2025}
}

@article{xue2023stability,
  title={Stability guarantees for feature attributions with multiplicative smoothing},
  author={Xue, Anton and Alur, Rajeev and Wong, Eric},
  journal={Advances in Neural Information Processing Systems},
  volume={36},
  pages={62388--62413},
  year={2023}
}

@article{goldwasser2024provably,
  title={Provably Stable Feature Rankings with SHAP and LIME},
  author={Goldwasser, Jeremy and Hooker, Giles},
  journal={arXiv preprint arXiv:2401.15800},
  year={2024}
}

@inproceedings{bassan2023towards,
  title={Towards formal XAI: formally approximate minimal explanations of neural networks},
  author={Bassan, Shahaf and Katz, Guy},
  booktitle={International Conference on Tools and Algorithms for the Construction and Analysis of Systems},
  pages={187--207},
  year={2023},
  organization={Springer}
}

@inproceedings{slack2020fooling,
  title={Fooling lime and shap: Adversarial attacks on post hoc explanation methods},
  author={Slack, Dylan and Hilgard, Sophie and Jia, Emily and Singh, Sameer and Lakkaraju, Himabindu},
  booktitle={Proceedings of the AAAI/ACM Conference on AI, Ethics, and Society},
  pages={180--186},
  year={2020}
}

@inproceedings{balasubramanian2023towards,
  title={Towards Improved Input Masking for Convolutional Neural Networks},
  author={Balasubramanian, Sriram and Feizi, Soheil},
  booktitle={Proceedings of the IEEE/CVF International Conference on Computer Vision},
  pages={1855--1865},
  year={2023}
}

@article{koyuncu2024exploiting,
  title={Exploiting the Data Gap: Utilizing Non-ignorable Missingness to Manipulate Model Learning},
  author={Koyuncu, Deniz and Gittens, Alex and Yener, B{\"u}lent and Yung, Moti},
  journal={arXiv preprint arXiv:2409.04407},
  year={2024}
}

@article{joe2022exploiting,
  title={Exploiting missing value patterns for a backdoor attack on machine learning models of electronic health records: Development and validation study},
  author={Joe, Byunggill and Park, Yonghyeon and Hamm, Jihun and Shin, Insik and Lee, Jiyeon and others},
  journal={JMIR Medical Informatics},
  volume={10},
  number={8},
  pages={e38440},
  year={2022},
  publisher={JMIR Publications Inc., Toronto, Canada}
}

@article{frye2020shapley,
  title={Shapley explainability on the data manifold},
  author={Frye, Christopher and de Mijolla, Damien and Begley, Tom and Cowton, Laurence and Stanley, Megan and Feige, Ilya},
  journal={arXiv preprint arXiv:2006.01272},
  year={2020}
}

@article{chirkova2023should,
  title={Should you marginalize over possible tokenizations?},
  author={Chirkova, Nadezhda and Kruszewski, Germ{\'a}n and Rozen, Jos and Dymetman, Marc},
  journal={arXiv preprint arXiv:2306.17757},
  year={2023}
}

@article{rong2022consistent,
  title={A consistent and efficient evaluation strategy for attribution methods},
  author={Rong, Yao and Leemann, Tobias and Borisov, Vadim and Kasneci, Gjergji and Kasneci, Enkelejda},
  journal={arXiv preprint arXiv:2202.00449},
  year={2022}
}

@article{park2024geometric,
  title={Geometric Remove-and-Retrain (GOAR): Coordinate-Invariant eXplainable AI Assessment},
  author={Park, Yong-Hyun and Seo, Junghoon and Park, Bomseok and Lee, Seongsu and Jo, Junghyo},
  journal={arXiv preprint arXiv:2407.12401},
  year={2024}
}

@article{bassan2025makes,
  title={What makes an Ensemble (Un) Interpretable?},
  author={Bassan, Shahaf and Amir, Guy and Zehavi, Meirav and Katz, Guy},
  journal={arXiv e-prints},
  pages={arXiv--2506},
  year={2025}
}

@article{havaldar2025t,
  title={T-FIX: Text-Based Explanations with Features Interpretable to eXperts},
  author={Havaldar, Shreya and Jin, Helen and Kim, Chaehyeon and Xue, Anton and You, Weiqiu and Gatti, Marco and Jain, Bhuvnesh and Qu, Helen and Hashimoto, Daniel A and Madani, Amin and others},
  journal={arXiv preprint arXiv:2511.04070},
  year={2025}
}

@article{bassan2025explaining,
  title={Explaining, fast and slow: Abstraction and refinement of provable explanations},
  author={Bassan, Shahaf and Elboher, Yizhak Yisrael and Ladner, Tobias and Althoff, Matthias and Katz, Guy},
  journal={arXiv preprint arXiv:2506.08505},
  year={2025}
}

@article{lin2023robustness,
  title={On the robustness of removal-based feature attributions},
  author={Lin, Chris and Covert, Ian and Lee, Su-In},
  journal={Advances in Neural Information Processing Systems},
  volume={36},
  pages={79613--79666},
  year={2023}
}

@inproceedings{shashikumar2023unsupervised,
  title={Unsupervised detection and correction of model calibration shift at test-time},
  author={Shashikumar, Supreeth P and Amrollahi, Fatemeh and Nemati, Shamim},
  booktitle={2023 45th Annual International Conference of the IEEE Engineering in Medicine \& Biology Society (EMBC)},
  pages={1--4},
  year={2023},
  organization={IEEE}
}

@inproceedings{tomani2021post,
  title={Post-hoc uncertainty calibration for domain drift scenarios},
  author={Tomani, Christian and Gruber, Sebastian and Erdem, Muhammed Ebrar and Cremers, Daniel and Buettner, Florian},
  booktitle={Proceedings of the IEEE/CVF Conference on Computer Vision and Pattern Recognition},
  pages={10124--10132},
  year={2021}
}

@inproceedings{kumar2022calibrated,
  title={Calibrated ensembles can mitigate accuracy tradeoffs under distribution shift},
  author={Kumar, Ananya and Ma, Tengyu and Liang, Percy and Raghunathan, Aditi},
  booktitle={Uncertainty in Artificial Intelligence},
  pages={1041--1051},
  year={2022},
  organization={PMLR}
}

@article{haug2021baselines,
  title={On baselines for local feature attributions},
  author={Haug, Johannes and Z{\"u}rn, Stefan and El-Jiz, Peter and Kasneci, Gjergji},
  journal={arXiv preprint arXiv:2101.00905},
  year={2021}
}

@article{vo2024explainability,
  title={Explainability of Machine Learning Models under Missing Data},
  author={Vo, Tuan L and Nguyen, Thu and Lopez-Ramos, Luis M and Hammer, Hugo L and Riegler, Michael A and Halvorsen, Pal},
  journal={arXiv preprint arXiv:2407.00411},
  year={2024}
}

@misc{nickparvar2021mri,
  title={Brain Tumor MRI Dataset},
  url={https://www.kaggle.com/dsv/2645886},
  DOI={10.34740/KAGGLE/DSV/2645886},
  publisher={Kaggle},
  author={Msoud Nickparvar},
  year={2021}
}

@inproceedings{pal2022medmcqa,
  title={Medmcqa: A large-scale multi-subject multi-choice dataset for medical domain question answering},
  author={Pal, Ankit and Umapathi, Logesh Kumar and Sankarasubbu, Malaikannan},
  booktitle={Conference on health, inference, and learning},
  pages={248--260},
  year={2022},
  organization={PMLR}
}

@article{jin2021disease,
  title={What disease does this patient have? a large-scale open domain question answering dataset from medical exams},
  author={Jin, Di and Pan, Eileen and Oufattole, Nassim and Weng, Wei-Hung and Fang, Hanyi and Szolovits, Peter},
  journal={Applied Sciences},
  volume={11},
  number={14},
  pages={6421},
  year={2021},
  publisher={MDPI}
}

@article{kingma2014adam,
  title={Adam: A method for stochastic optimization},
  author={Kingma, Diederik P and Ba, Jimmy},
  journal={arXiv preprint arXiv:1412.6980},
  year={2014}
}

@inproceedings{chen2016xgboost,
  title={Xgboost: A scalable tree boosting system},
  author={Chen, Tianqi and Guestrin, Carlos},
  booktitle={Proceedings of the 22nd acm sigkdd international conference on knowledge discovery and data mining},
  pages={785--794},
  year={2016}
}

@article{adebayo2018sanity,
  title={Sanity checks for saliency maps},
  author={Adebayo, Julius and Gilmer, Justin and Muelly, Michael and Goodfellow, Ian and Hardt, Moritz and Kim, Been},
  journal={Advances in neural information processing systems},
  volume={31},
  year={2018}
}

@inproceedings{adebayo2022post,
  title={Post hoc explanations may be ineffective for detecting unknown spurious correlation},
  author={Adebayo, Julius and Muelly, Michael and Abelson, Harold and Kim, Been},
  booktitle={International conference on learning representations},
  year={2022}
}

@article{agarwal2022openxai,
  title={Openxai: Towards a transparent evaluation of model explanations},
  author={Agarwal, Chirag and Krishna, Satyapriya and Saxena, Eshika and Pawelczyk, Martin and Johnson, Nari and Puri, Isha and Zitnik, Marinka and Lakkaraju, Himabindu},
  journal={Advances in neural information processing systems},
  volume={35},
  pages={15784--15799},
  year={2022}
}

@article{dinu2020challenging,
  title={Challenging common interpretability assumptions in feature attribution explanations},
  author={Dinu, Jonathan and Bigham, Jeffrey and Kolter, J Zico},
  journal={arXiv preprint arXiv:2012.02748},
  year={2020}
}

@article{jin2024fix,
  title={The FIX Benchmark: Extracting Features Interpretable to eXperts},
  author={Jin, Helen and Havaldar, Shreya and Kim, Chaehyeon and Xue, Anton and You, Weiqiu and Qu, Helen and Gatti, Marco and Hashimoto, Daniel A and Jain, Bhuvnesh and Madani, Amin and others},
  journal={arXiv preprint arXiv:2409.13684},
  year={2024}
}

@article{hurst2024gpt,
  title={Gpt-4o system card},
  author={Hurst, Aaron and Lerer, Adam and Goucher, Adam P and Perelman, Adam and Ramesh, Aditya and Clark, Aidan and Ostrow, AJ and Welihinda, Akila and Hayes, Alan and Radford, Alec and others},
  journal={arXiv preprint arXiv:2410.21276},
  year={2024}
}

@article{kindermans2019reliability,
  title={The (un) reliability of saliency methods},
  author={Kindermans, Pieter-Jan and Hooker, Sara and Adebayo, Julius and Alber, Maximilian and Sch{\"u}tt, Kristof T and D{\"a}hne, Sven and Erhan, Dumitru and Kim, Been},
  journal={Explainable AI: Interpreting, explaining and visualizing deep learning},
  pages={267--280},
  year={2019},
  publisher={Springer}
}

@article{nauta2023anecdotal,
  title={From anecdotal evidence to quantitative evaluation methods: A systematic review on evaluating explainable ai},
  author={Nauta, Meike and Trienes, Jan and Pathak, Shreyasi and Nguyen, Elisa and Peters, Michelle and Schmitt, Yasmin and Schl{\"o}tterer, J{\"o}rg and Van Keulen, Maurice and Seifert, Christin},
  journal={ACM Computing Surveys},
  volume={55},
  number={13s},
  pages={1--42},
  year={2023},
  publisher={ACM New York, NY}
}

@misc{campos2000cardiotocography,
  author       = {Campos, D. and Bernardes, J.},
  title        = {{Cardiotocography}},
  year         = {2000},
  howpublished = {UCI Machine Learning Repository},
  note         = {{DOI}: https://doi.org/10.24432/C51S4N}
}

@misc{wolberg1993breast,
  author       = {Wolberg, William and Mangasarian, Olvi and Street, Nick and Street, W.},
  title        = {{Breast Cancer Wisconsin (Diagnostic)}},
  year         = {1993},
  howpublished = {UCI Machine Learning Repository},
  note         = {{DOI}: https://doi.org/10.24432/C5DW2B}
}

@inproceedings{zhou2022feature,
  title={Do feature attribution methods correctly attribute features?},
  author={Zhou, Yilun and Booth, Serena and Ribeiro, Marco Tulio and Shah, Julie},
  booktitle={Proceedings of the AAAI conference on artificial intelligence},
  volume={36},
  pages={9623--9633},
  year={2022}
}

@article{you2025probabilistic,
  title={Probabilistic Soundness Guarantees in LLM Reasoning Chains},
  author={You, Weiqiu and Xue, Anton and Havaldar, Shreya and Rao, Delip and Jin, Helen and Callison-Burch, Chris and Wong, Eric},
  journal={arXiv preprint arXiv:2507.12948},
  year={2025}
}

@article{hu2022lora,
  title={Lora: Low-rank adaptation of large language models.},
  author={Hu, Edward J and Shen, Yelong and Wallis, Phillip and Allen-Zhu, Zeyuan and Li, Yuanzhi and Wang, Shean and Wang, Lu and Chen, Weizhu and others},
  journal={ICLR},
  volume={1},
  number={2},
  pages={3},
  year={2022}
}

@article{carter2021overinterpretation,
  title={Overinterpretation reveals image classification model pathologies},
  author={Carter, Brandon and Jain, Siddhartha and Mueller, Jonas W and Gifford, David},
  journal={Advances in Neural Information Processing Systems},
  volume={34},
  pages={15395--15407},
  year={2021}
}

@article{decker2025improving,
  title={Improving Perturbation-based Explanations by Understanding the Role of Uncertainty Calibration},
  author={Decker, Thomas and Tresp, Volker and Buettner, Florian},
  journal={arXiv preprint arXiv:2511.10439},
  year={2025}
}
\bibliographystyle{iclr2026_conference}

\newpage
\appendix

\section{Additional Experiments and Details}
\label{sec:additional_experiments}

We present our experimental setup here, along with any additional experiments and relevant details.

\paragraph{Compute.}
We had access to a server with four NVIDIA H100 NVL GPUs.

\subsection{Metrics for Feature Attributions}
\label{sec:other_missingness_metrics}

Given an input \(x \in \mathbb{R}^n\) and a classifier \(f: \mbb{R}^n \to \mbb{R}^m\), a feature attribution explanation method returns a vector \(\alpha \in \mbb{R}^n\) where \(\alpha_i\) denotes the importance of feature \(x_i\).
We now discuss some metrics for evaluating the quality of the feature attribution \(\alpha\).

\paragraph{Sufficiency.}
One way to assess the quality of an attribution \(\alpha\) is by using \text{only} its top-k-selected features for prediction.
Let \(\msf{Top}_k (x, \alpha) \in \mbb{R}^n\) be the version of \(x\) where its top-k features are selected, as ranked by \(\alpha\).
Equivalently, \(\msf{Top}_k (x, \alpha)\) is the version of \(x\) where its bottom \(n-k\) features are ablated.
Then, the top-k sufficiency metric~\citep{hase2021out} is defined as:
\begin{equation}
    \msf{Sufficiency}(f, x, k) = f(x)_{\hat{y}} - f(\msf{Top}_k(x, \alpha))_{\hat{y}},
\end{equation}
where \(\hat{y} = \arg\max_{y} f(x)_y\) is the predicted class.
In this formulation, a lower sufficiency is preferable, as it indicates that the selected features can more reliably attain the confidence associated with a clean input's prediction.

\paragraph{Sensitivity.}
Conversely, one can also assess how much \textit{omitting} the top-k features would affect prediction.
Analogously to the sufficiency metric, let \(\msf{Bot}_{n-k} (x, \alpha)\) denote the version of \(x\) where the bottom \(n-k\) features are selected; i.e., the top \(k\) features are ablated.
The top-k sensitivity metric, also called \textit{comprehensiveness}~\citep{hase2021out}, is defined as:
\begin{equation}
    \text{Sensitivity}(f, x, k) = f(x)_{\hat{y}} - f(\msf{Bot}_{n-k} (x, \alpha))_{\hat{y}}
\end{equation}
A higher score indicates the features were critical to the prediction, whereas a lower score suggests that the model is more robust (i.e., less sensitive) to their inclusion.




{\color{rev1}

\begin{figure*}[t]

\centering

\begin{minipage}{0.32\linewidth}
    \includegraphics[width=1.0\linewidth]{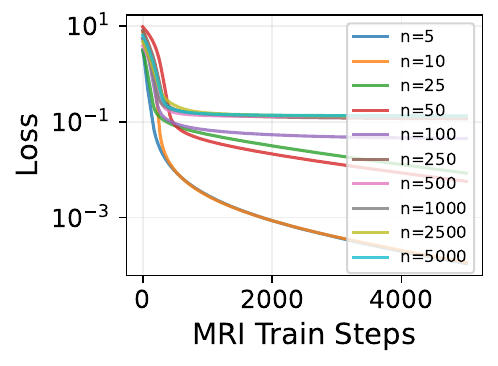}
\end{minipage}%
\begin{minipage}{0.32\linewidth}
    \includegraphics[width=1.0\linewidth]{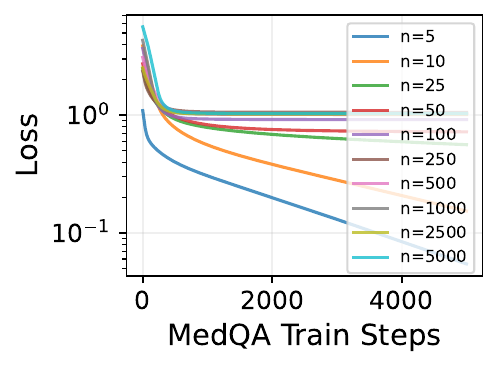}
\end{minipage}%
\begin{minipage}{0.32\linewidth}
    \includegraphics[width=1.0\linewidth]{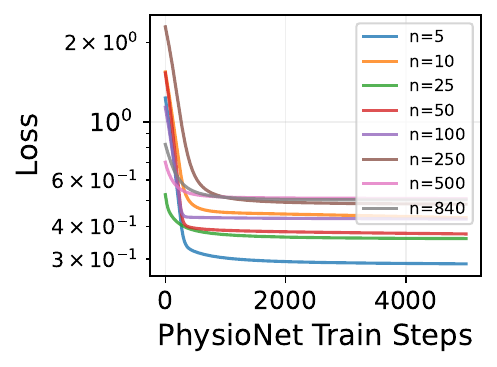}
\end{minipage}%

\begin{minipage}{0.32\linewidth}
    \includegraphics[width=1.0\linewidth]{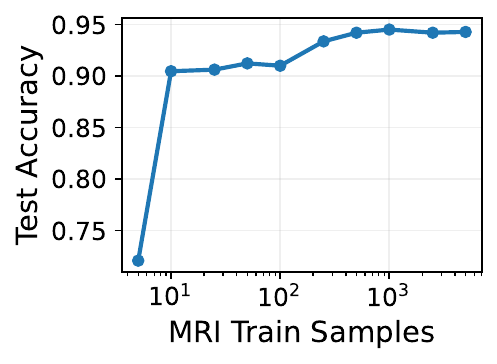}
\end{minipage}%
\begin{minipage}{0.32\linewidth}
    \includegraphics[width=1.0\linewidth]{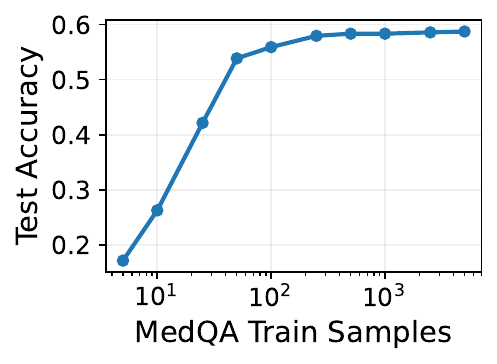}
\end{minipage}%
\begin{minipage}{0.32\linewidth}
    \includegraphics[width=1.0\linewidth]{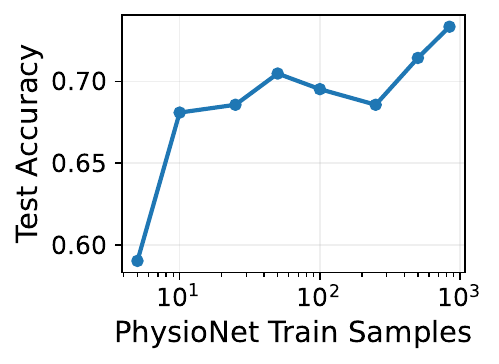}
\end{minipage}%

\begin{minipage}{0.32\linewidth}
    \includegraphics[width=1.0\linewidth]{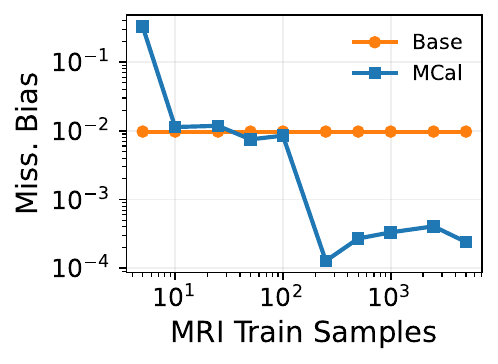}
\end{minipage}%
\begin{minipage}{0.32\linewidth}
    \includegraphics[width=1.0\linewidth]{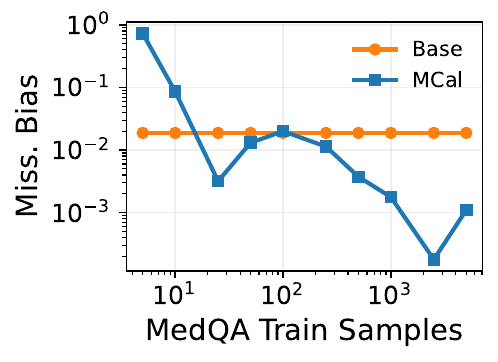}
\end{minipage}%
\begin{minipage}{0.32\linewidth}
    \includegraphics[width=1.0\linewidth]{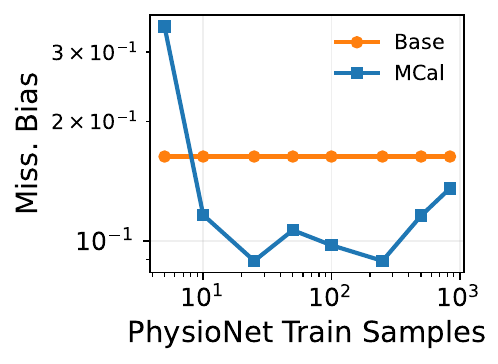}
\end{minipage}%

\begin{minipage}{0.32\linewidth}
    \includegraphics[width=1.0\linewidth]{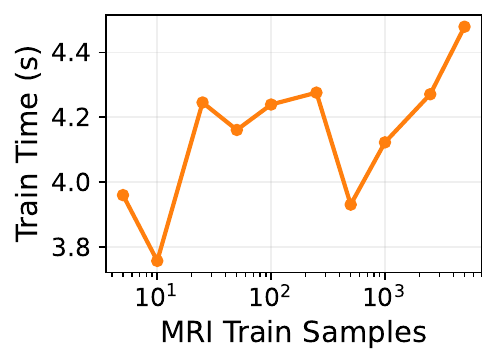}
\end{minipage}%
\begin{minipage}{0.32\linewidth}
    \includegraphics[width=1.0\linewidth]{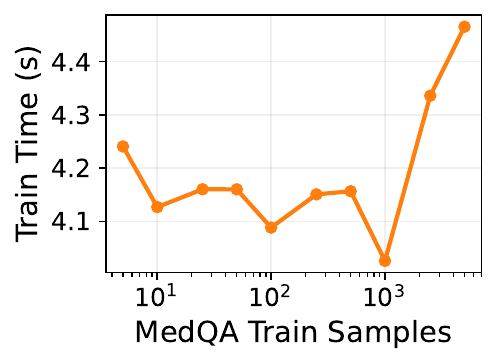}
\end{minipage}%
\begin{minipage}{0.32\linewidth}
    \includegraphics[width=1.0\linewidth]{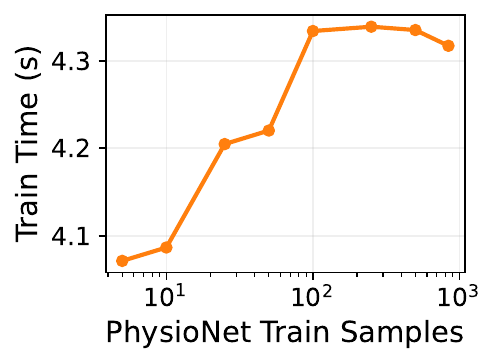}
\end{minipage}%

\caption{{\color{rev1}
\textbf{MCal training dynamics and performance on different dataset sizes.}
For different amounts (\(n\)) of clean-ablated input pairs, we show the training loss curves and test-set accuracies.
Each training run consisted of \(5000\) iterations, and all runs finished in \(\leq 5\) seconds.
}
}
\label{fig:mcal_training_dynamics}
\end{figure*}

\subsection{MCal Training Dynamics}
\label{sec:mcal_training_dynamics}

Here, we investigate the training dynamics and performance of MCal as the training set size varies.
We show the results in~\cref{fig:mcal_training_dynamics}.
In general, as the training dataset size increases, test-time accuracy increases.
When \(n\) is small, however, the problem is over-parametrized, meaning that the training loss continues to decrease without significantly improving test-time accuracy.
}

\subsection{MCal vs. Retrain}

We note an interesting phenomenon in~\cref{tab:main_paper_baselines}, where MCal often outperforms retrain-based methods.
At first glance, this is surprising because the model's representational capacity should exceed that of MCal.
We suspect that this may be due to our use of \textit{conditioned} MCal to train an ensemble of calibrator heads, which may at times be sufficient to overcome this gap in representational power.
Indeed, we observe that \textit{unconditioned} MCal, in which only a single calibrator head is used across all missingness levels, typically underperforms retraining.

In attempting to diagnose this observation with more intensive training runs (25K steps, $10^{-3}$ learning rate and 50 ablations per data point) we observe that while retrain can outperform MCal in certain situations, its unconditioned variant may also do so.

We suspect this may be due to the differences between the formulations of missingness bias in~\cref{eqn:mb_def} and the optimization objective in~\cref{eqn:finetune_loss}.
A more in-depth analysis would be warranted in future work. We show these trends in~\cref{fig:retrain_analysis}, with a focus on tabular models given their ease of retraining.



\begin{figure}[t]

\centering

\begin{minipage}{0.32\linewidth}
    \includegraphics[width=1.0\linewidth]{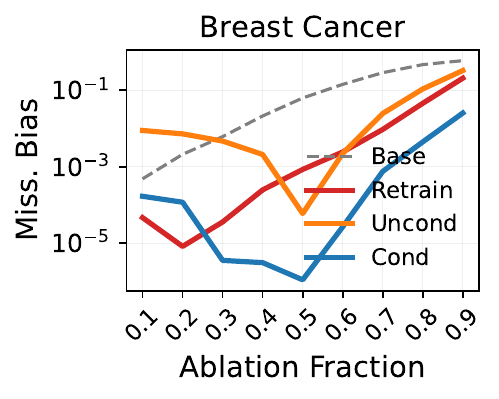}
\end{minipage}%
\begin{minipage}{0.32\linewidth}
    \includegraphics[width=1.0\linewidth]{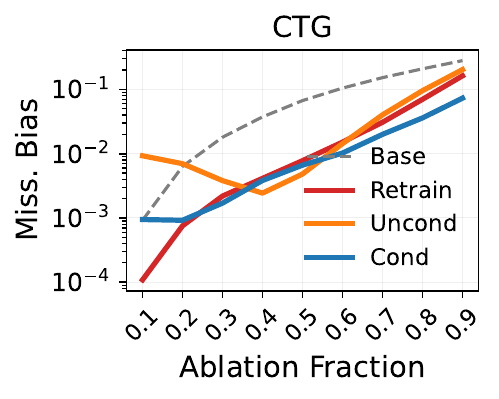}
\end{minipage}%
\begin{minipage}{0.32\linewidth}
    \includegraphics[width=1.0\linewidth]{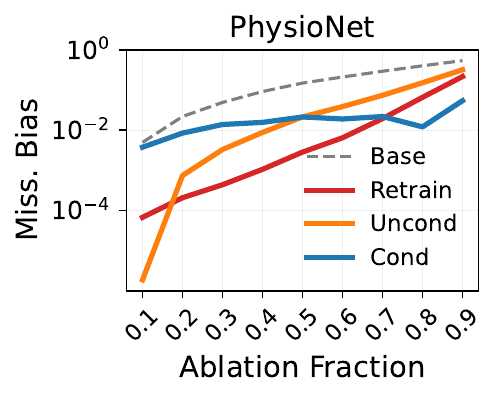}
\end{minipage}%

\caption{\color{rev1}\textbf{MCal can outperform retrain due to its specialization}
Fitting an ensemble of calibrators at a discretized set of ablation rates can even outperform retraining the model at times
(Left) Breast Cancer, (Middle) CTG, (Right) PhysioNet.
}
\label{fig:retrain_analysis}
\end{figure}

{\color{rev1}
\begin{figure*}[t]

\centering

\includegraphics[width=0.9\linewidth]{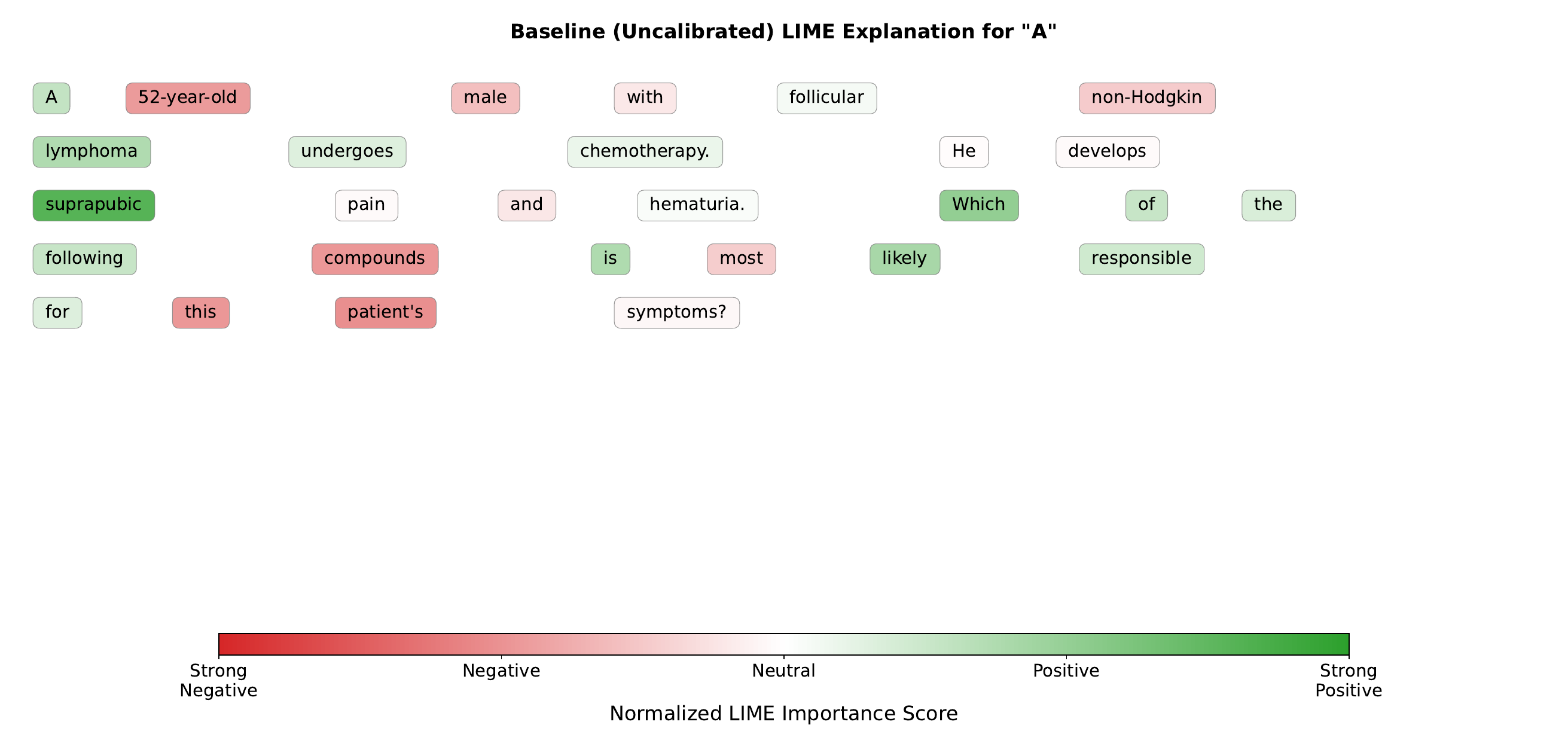}

\includegraphics[width=0.9\linewidth]{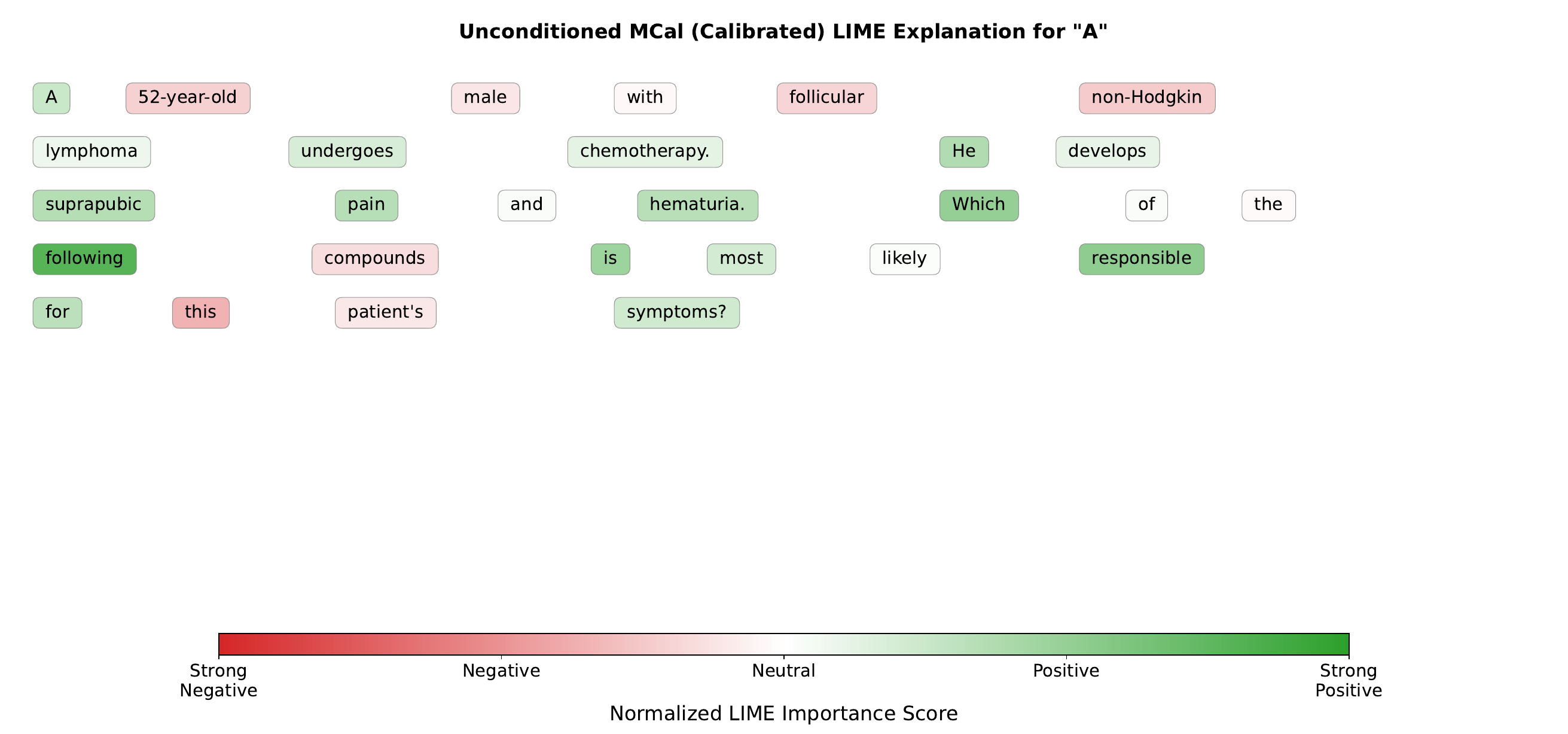}

\caption{
{\color{rev1}
\textbf{For a selected example question, MCal results in different feature importances}, for example medically relevant features/terms such as "hematuria" gain importance in the calibrated heatmap. The model task is, for the above question, to choose between the Options: A: Cyclophosphamide, B: Cisplatin, C: Mesna, D: Bleomycin.
}
}
\label{fig:openai_lime_heatmap_comparison}
\end{figure*}

\begin{figure*}[t]

\centering

\begin{minipage}{0.32\linewidth}
    \includegraphics[width=1.0\linewidth]{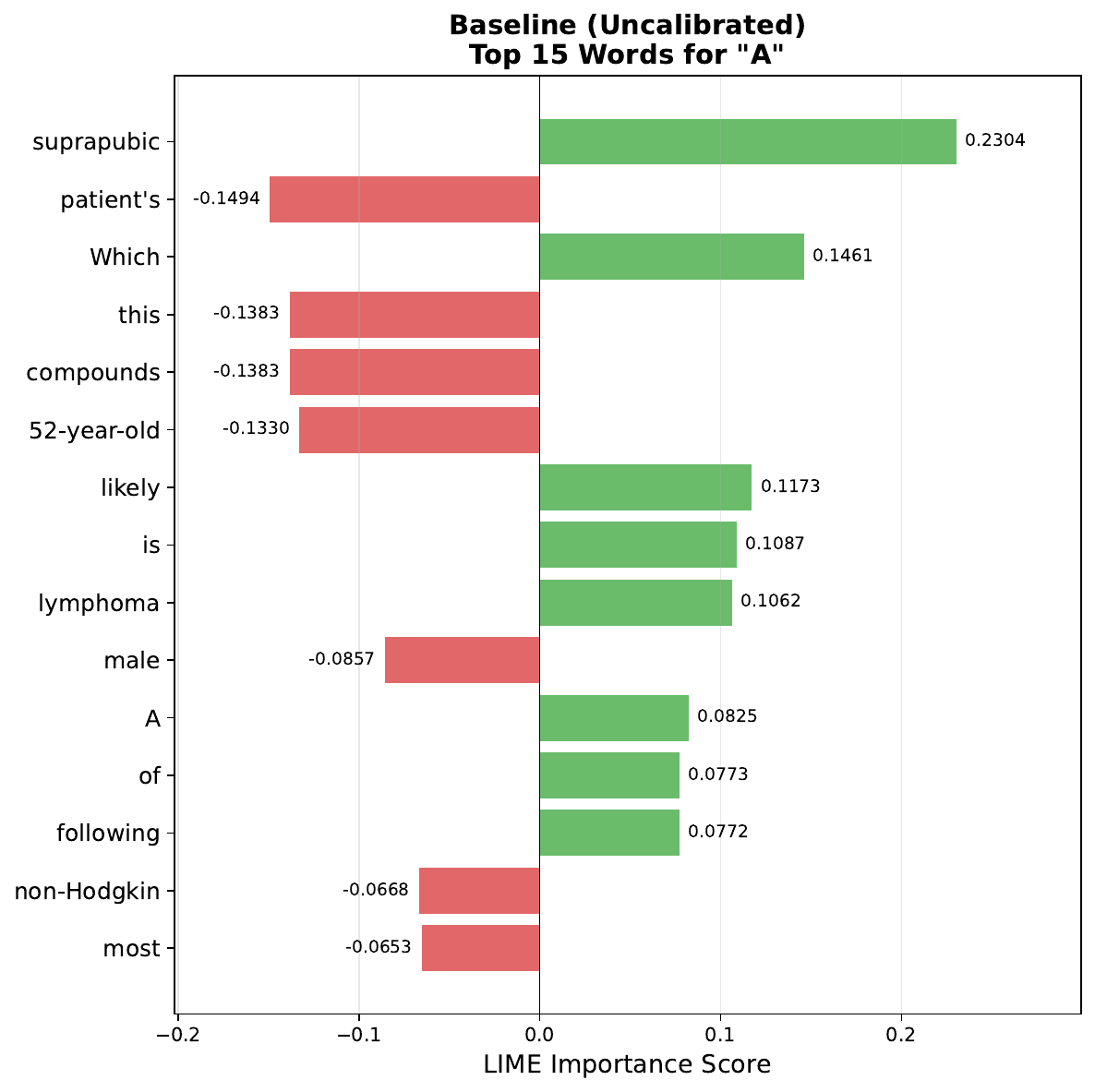}
\end{minipage}%
\begin{minipage}{0.32\linewidth}
    \includegraphics[width=1.0\linewidth]{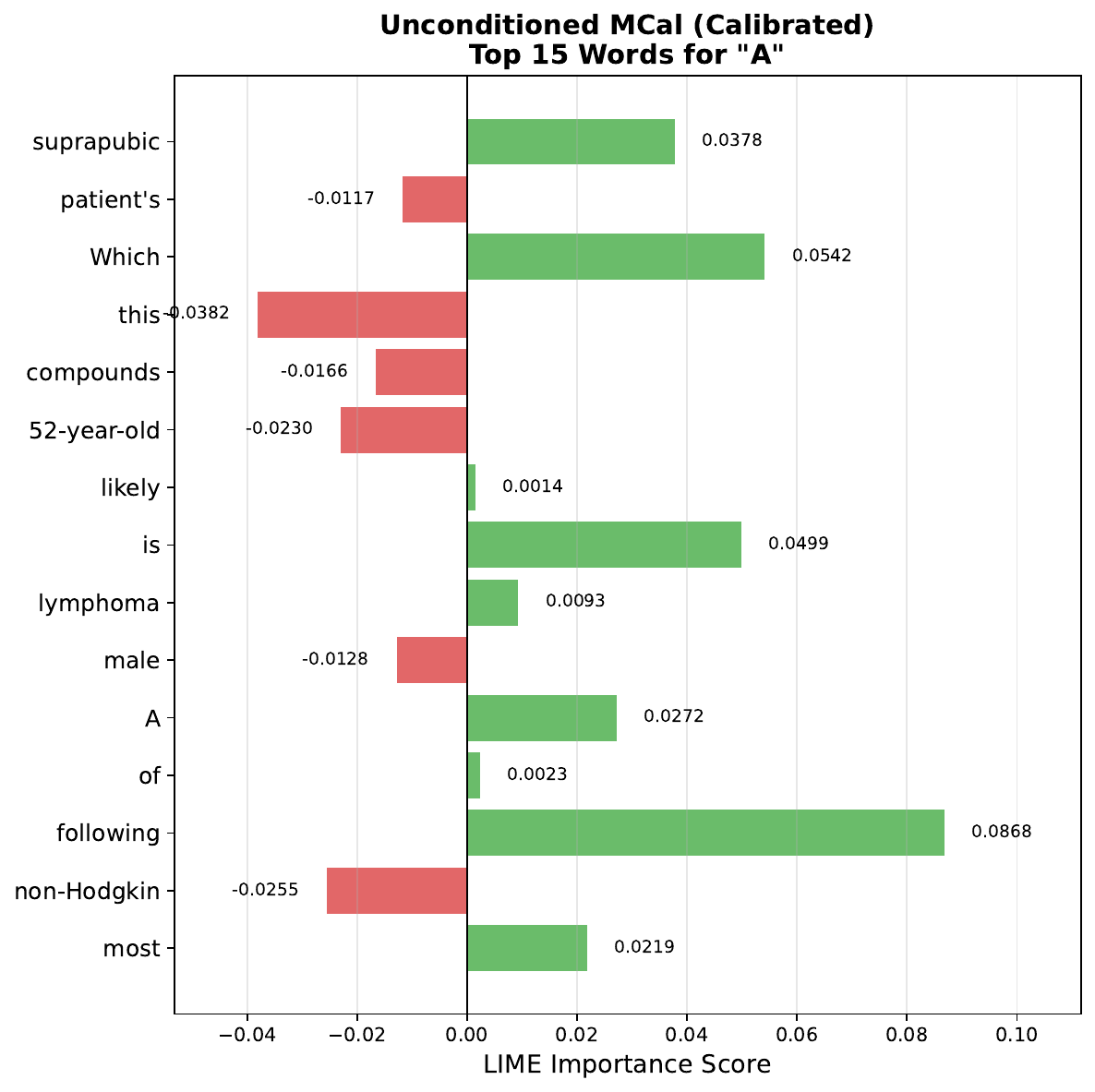}
\end{minipage}%
\begin{minipage}{0.32\linewidth}
    \includegraphics[width=1.0\linewidth]{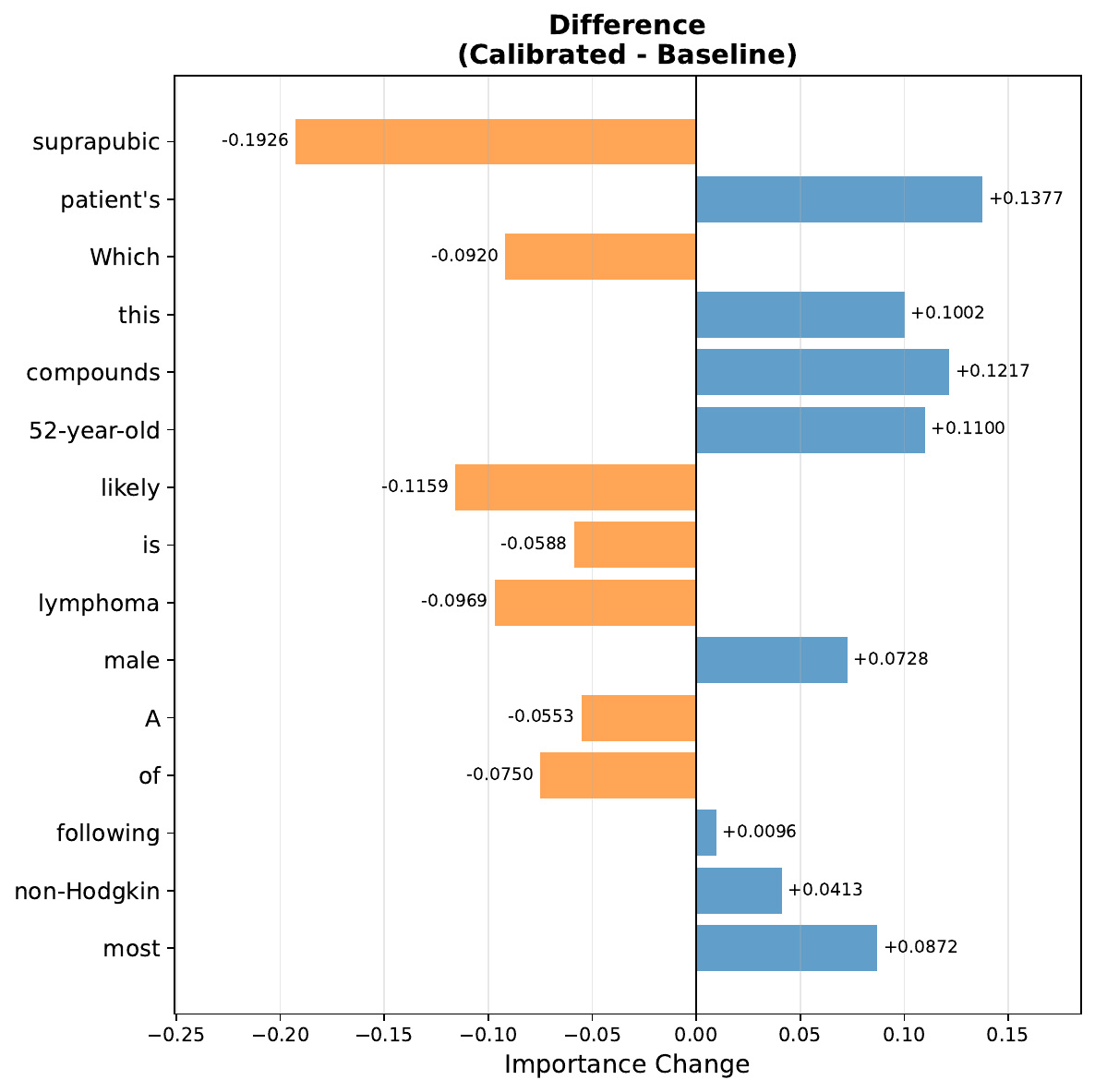}
\end{minipage}%

\caption{{\color{rev1}Direct comparison of LIME-derived feature importance values, with (a) Uncalibrated, (b) Calibrated, and (c) demonstrating the difference between the two, i.e., (b) - (a).
}
}
\label{fig:openai_barchart_comparison}
\end{figure*}

\subsection{Case Study: Integration with API-based Models}

Current-day machine learning workflows are increasingly dependent on closed-weight API-based models.
Recognizing this, we perform a case study demonstrating how MCal can be extended to such settings, assuming only access to output logits.



In particular, we demonstrate how GPT-4o-mini~\citep{hurst2024gpt} can be calibrated on selected MedQA instances in~\cref{fig:openai_lime_heatmap_comparison} and ~\cref{fig:openai_barchart_comparison}.
The calibrated model redistributes feature importance in potentially meaningful ways, elevating diagnostic symptoms like ``hematuria''(a hallmark sign of cyclophosphamide toxicity) and ``pain'', while reducing dominance of the anatomical term ``suprapubic''.
This rebalancing suggests calibration may produce more clinically-aligned and faithful explanations, though domain expert validation and further faithfulness testing are needed.

\label{sec:api_based_models}
} 




\begin{figure*}[t]

\centering

\includegraphics[width=0.9\linewidth]{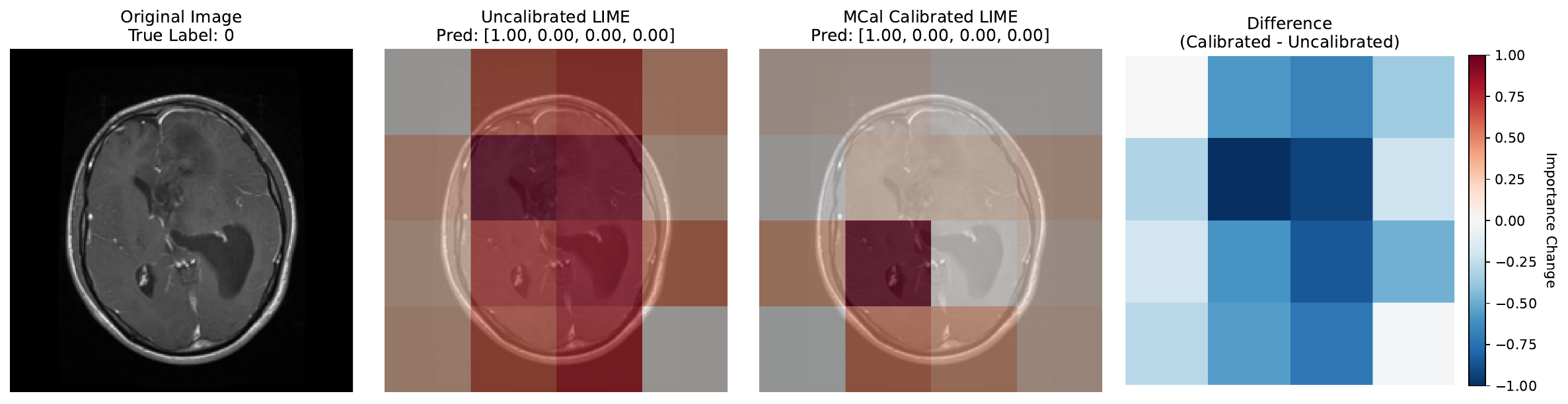}

\includegraphics[width=0.9\linewidth]{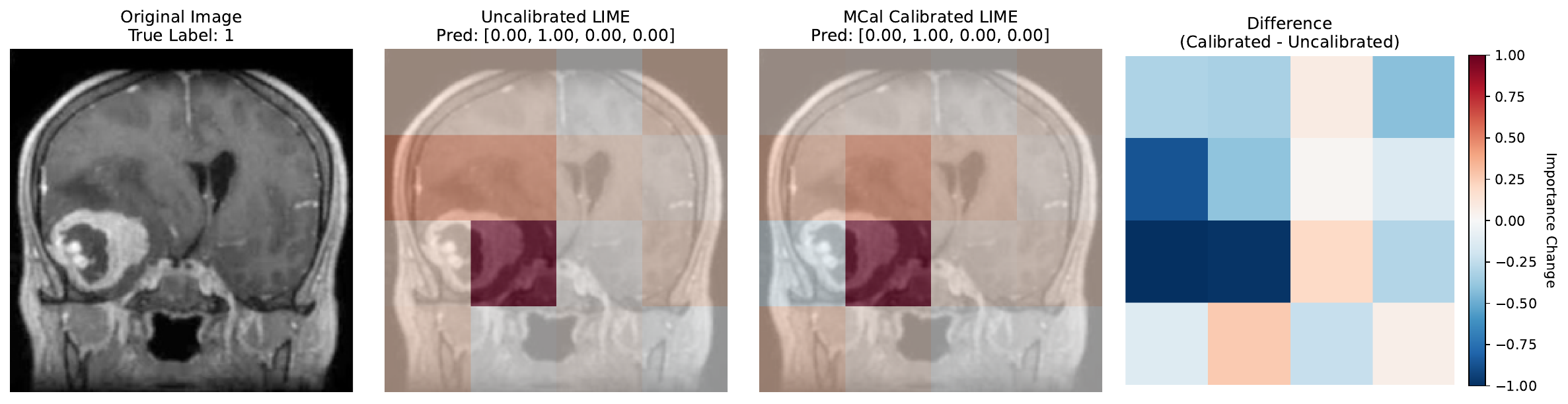}

\includegraphics[width=0.9\linewidth]{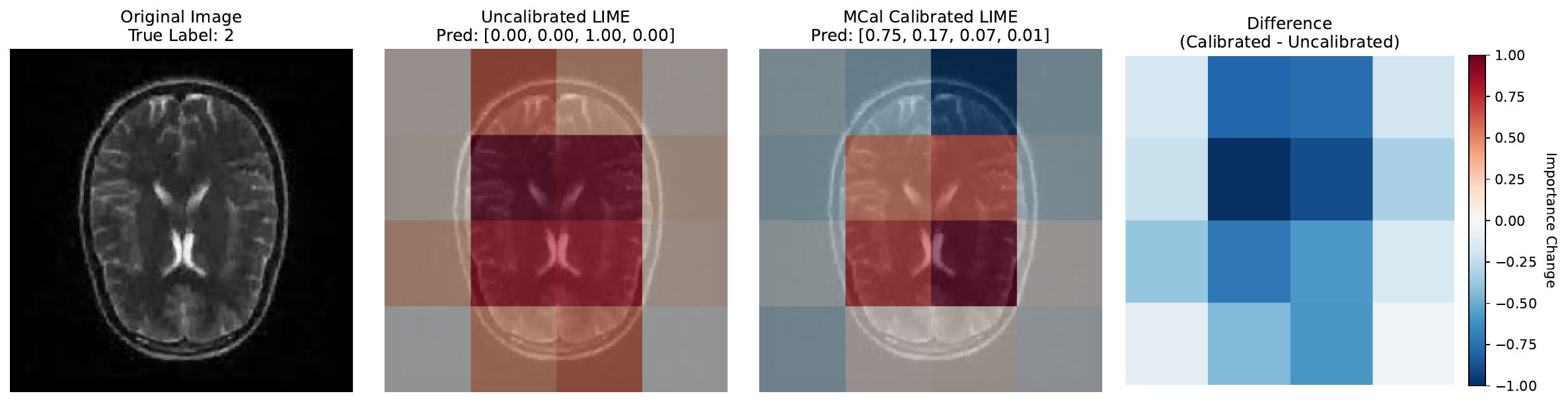}

\includegraphics[width=0.9\linewidth]{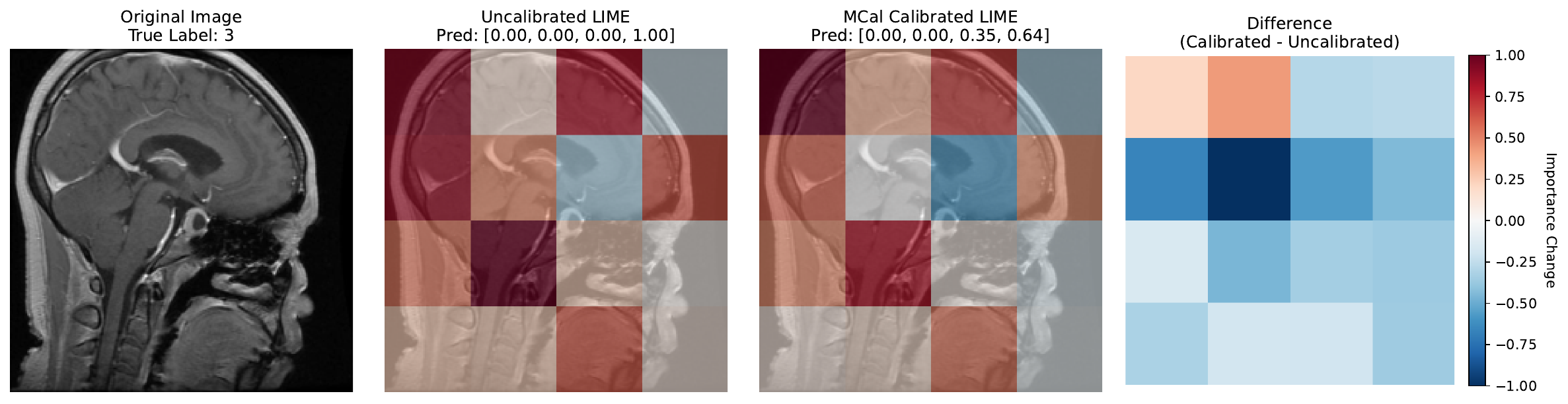}

\caption{
{\color{rev1}
\textbf{Selected examples of LIME on the MRI dataset.}
In calibrated models, we observe that LIME tends to assign less importance to border patches, where relevant features are less likely to occur.
The four classes are: Meningioma (0), Glioma (1), Pituitary Tumor (2), and No Tumor (3).
}
}
\label{fig:mri_lime_visual_examples}
\end{figure*}

\end{document}